\documentclass{article}

\usepackage{amsmath,amsfonts,amssymb,bm}
\usepackage{graphics,float}
\newcommand{\scalar}[2]{\langle #1 , #2 \rangle}

\def\1{\bm{1}}

\def\vf{{\bm{f}}}

\def\vh{{\bm{h}}}

\def\vu{{\bm{u}}}
\def\vv{{\bm{v}}}

\def\mC{{\bm{C}}}

\def\mF{{\bm{F}}}

\def\mP{{\bm{P}}}

\def\mT{{\bm{T}}}

\DeclareMathAlphabet{\mathsfit}{\encodingdefault}{\sfdefault}{m}{sl}
\SetMathAlphabet{\mathsfit}{bold}{\encodingdefault}{\sfdefault}{bx}{n}

\def\sA{{\mathbb{A}}}

\def\sR{{\mathbb{R}}}
\def\sS{{\mathbb{S}}}

\newcommand{\R}{\mathbb{R}}

\PassOptionsToPackage{numbers, compress}{natbib}
\usepackage[preprint]{neurips_2022}

\usepackage[utf8]{inputenc} 
\usepackage[T1]{fontenc}    
\usepackage{hyperref}       
\hypersetup{colorlinks,allcolors=black}
\usepackage{booktabs}       
\usepackage{amsfonts}       
\usepackage{nicefrac}       
\usepackage{microtype}      
\usepackage{wrapfig}
\usepackage{pdfpages}
\usepackage{amsthm}

\usepackage{fancyhdr}
\usepackage{graphicx,subfigure}

\newtheorem{lemma}{Lemma}

\usepackage[framemethod=TikZ]{mdframed}

\newcommand{\FGW}{\operatorname{FGW}}
\newcommand{\TFGW}{\operatorname{TFGW}}

\newcommand{\integ}[1]{{[\![#1]\!]}}

\definecolor{myred}{HTML}{dd3c3c}
\title{Template based Graph Neural Network with \\ Optimal Transport Distances}

\author{C\'{e}dric Vincent-Cuaz  \textsuperscript{1}, R\'{e}mi Flamary \textsuperscript{2}, Marco Corneli \textsuperscript{1,3}, Titouan Vayer \textsuperscript{4}, Nicolas Courty \textsuperscript{5}  \\
	Univ. C{\^o}te d{'}Azur, Inria, Maasai, CNRS, LJAD \textsuperscript{1}; IP Paris, CMAP, UMR 7641 \textsuperscript{2}; MSI \textsuperscript{3};\\ Univ. Lyon, Inria, CNRS, ENS de Lyon, LIP UMR 5668 \textsuperscript{4}; Univ. Bretagne-Suf, CNRS, IRISA \textsuperscript{5}. \\
	\texttt{\{cedric.vincent-cuaz; marco.corneli; titouan.vayer\}@inria.fr} \\
	\texttt{remi.flamary@polytechnique.edu};   \texttt{nicolas.courty@irisa.fr}
}

\begin{document}

	\maketitle
	\begin{abstract}
		Current Graph Neural Networks (GNN) architectures generally rely on two important components: node features embedding through message passing, and aggregation with a specialized form of pooling. The structural (or topological) information is implicitly taken into account in these two steps. We propose in this work a novel point of view, which places distances to some learnable graph templates at the core of the graph representation. This distance embedding is constructed thanks to an optimal transport distance: the Fused Gromov-Wasserstein (FGW) distance, which encodes simultaneously feature and structure dissimilarities by solving a soft graph-matching problem. We postulate that the vector of FGW distances to a set of template graphs has a strong discriminative power, which is then fed to a non-linear classifier for final predictions. Distance embedding can be seen as a new layer, and can leverage on existing message passing techniques to promote sensible feature representations. Interestingly enough, in our work the optimal set of template graphs is also learnt in  an end-to-end fashion by differentiating through this layer. After describing the corresponding learning procedure, we empirically validate our claim on several synthetic and real life graph classification datasets, where our method is competitive or surpasses kernel and GNN state-of-the-art approaches. We complete our experiments by an ablation study and a sensitivity analysis to parameters.
	\end{abstract}
	\allowdisplaybreaks
	\section{Introduction}
	

Attributed graphs are characterized by {\em i)} the relationships between the nodes of the graph (structural or topological information) and {\em ii)} some specific features or attributes endowing the nodes themselves. Learning from those data is ubiquitous in many research areas
\citep{battaglia2018relational}, {\em e.g.} image analysis \citep{harchaoui2007image, bronstein2017geometric}, brain connectivity
\cite{ktena-distance-2017}, biological compounds \cite{jumper2021highly} or
social networks \cite{yanardag-deep-2015}, to name a few. Various methodologies approach the inherent complexity of those data, such as signal processing \cite{Shuman2013TheEF}, Bayesian and kernel methods on graphs \cite{NEURIPS2018_1fc21400, kriege2020survey} or more recently Graph Neural Networks (GNN) \cite{wu2020comprehensive} in the framework of the geometric deep learning
\cite{bronstein2017geometric, Bronstein2021GeometricDL}.

We are interested in this work in the classification of attributed graphs {\bf at the instance level}.
One existing approach consists in designing kernels that leverage topological properties of the observed graphs
\cite{borgwardt2005shortest, feragen2013scalable,gartner2003graph,
shervashidze2009efficient}. For instance, the popular
Weisfeiler-Lehman (WL) kernel \cite{shervashidze2011weisfeiler} iteratively aggregates for each node the features of its $k$-hop neighborhood. Alternative approaches aim at learning vectorial
representations of the graphs that can encode the graph structure (\emph{i.e. graph
representation learning} \cite{Chami2022}). 
In this domain, GNN lead to state-of-the-art performances with 
end-to-end learnable embeddings \cite{wu2020comprehensive}. 
At a given layer, these architectures typically
learn the node embeddings via local permutation-invariant transformations
aggregating its neighbour features~\cite{maron2018invariant,kipf2016semi, hamilton2017inductive,
xu2018powerful}. In order to obtain a representation of the whole graph suitable for classification, GNNs finally operate a pooling \cite{knyazev2019understanding, mesquita2020rethinking} of
the node embeddings, either global (\textit{e.g} summation over nodes
\cite{xu2018powerful}), or hierarchical (\textit{e.g} by iteratively clustering nodes \cite{zhang2018end,ying2018hierarchical,  lee2019self}). B

Another line of works targets the construction of meaningful distances that integrate simultaneously the structural and feature information, and that are based on optimal transport (OT) \cite{villani2009optimal, peyre-computational-2020}. Originally designed to compare probability distributions based on a geometric notion of optimality, it allows defining very general loss functions between various objects, modeled as probability distributions. In a nutshell, it proceeds by constructing a \emph{coupling} between the distributions that minimizes a specific \emph{cost}. 
Some approaches dealing with graphs rely on non-parametric models that first embed the graphs into a vectorial space and then match them via OT \cite{NikolentzosMV17, Togninalli19, kolouri2020wasserstein, maretic2019got}. Recently \citep{chen2020optimal} proposed the OT-GNN model, that embeds a graph as a vector of the Wasserstein distances between the nodes' embeddings (after GNN pre-processing) and learnt point clouds, acting as templates.

Building further from OT variants, the Gromov-Wasserstein (GW) distance~\cite{memoli2011gromov} directly handles graphs through the symmetric matrix $\mC$ that encodes the distance/similarity between every pairs of nodes (\textit{e.g.} adjacency, shortest path) and a \emph{weight vector} $\vh$ on the nodes encoding the nodes' relative importance. GW has proven to be useful for  tasks such as graph matching and partitioning \cite{xu2019scalable,chowdhury2021generalized} or unsupervised graph dictionary learning \cite{xu2020gromov, vincent2021online, vincent-cuaz2022semirelaxed}. GW has been
also extended to directed graphs \cite{chowdhury-gromov-wasserstein-2019} and to
attributed graphs via the Fused Gromov-Wasserstein (FGW) distance
\cite{titouan2019optimal, vayer2020fused}, that realizes a trade-off between an OT distance with a cost on node features and the GW distance between the similarity
matrices. Despite its recent successes on complex unsupervised tasks such as graph clustering \cite{xu2020gromov, vincent2021online, vincent-cuaz2022semirelaxed}, FGW has never been explored as part of an end-to-end model for graph
classification. In this work, we fill this gap by introducing a novel ``layer'' that embeds an attributed graph into a vector, whose coordinates are FGW distances to few (learned) graph templates.  While FGW can be performed directly on raw data (\textit{i.e.} the input structured graph without any pre-processing), we also consider the case where features representations are learnt from a GNN, similarly to OT-GNN \cite{chen2020optimal}, and thus also realizing a particular type of aggregation.

\begin{figure}[t!]
    \begin{center}
    \includegraphics[width=0.9\linewidth]{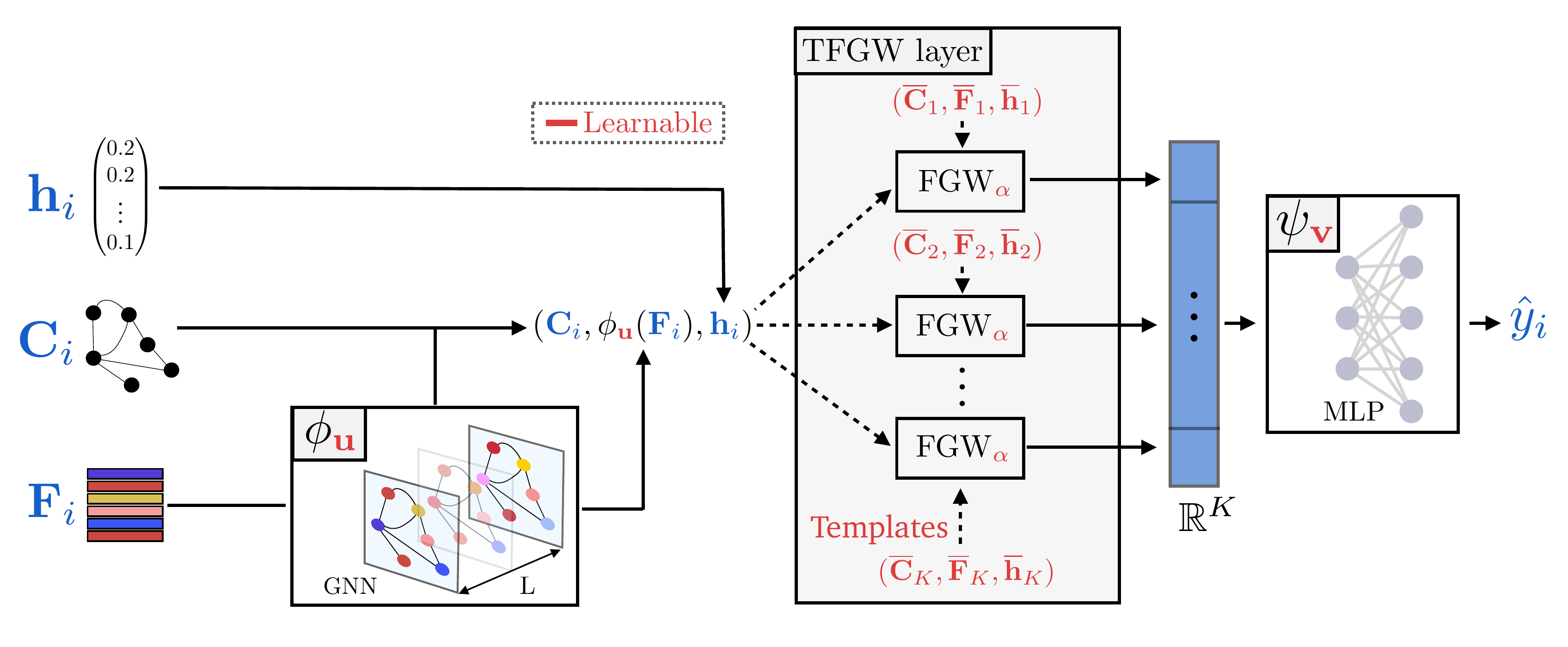}
    \end{center}
    \vspace{-3mm}
    \caption{\label{fig:schema}Illustration of the proposed model. \textbf{(left)} The input graph is represented as
    a triplet $(\mC_i,\mF_i,\vh_i)$ where the matrix $\mC_i$ encodes the structure, $\mF_i$ the features, $\vh_i$ the nodes' weights. A GNN $ \phi_{\vu}$ is applied to the raw features in order
    to extract a meaningful node representations. \textbf{(center)} The
    TFGW layer is applied to the filtered graph and provides a vector representation as FGW distances to templates. \textbf{(right)} a final MLP $\psi_{\vv}$ is
    applied to this vector in order to predict the final output
    of the model. All objects in \textcolor{myred}{red}  are parameters that are
    learned from the data. \vspace{-4mm}}
\end{figure}

\paragraph{Contributions.} We introduce a new GNN layer, named \textbf{TFGW} for
\textbf{Template-based FGW} and illustrated in the center of Figure \ref{fig:schema}. From
an input graph, it computes a vector of FGW distances to learnable graph
templates. This layer that can
be seen as an alternative to global pooling layers and can be
integrated into any neural network architecture.
We discuss its properties and the associated invariances. We detail the
optimization strategy that enables learning simultaneously GNN pre-processing
layers and graph
templates relevant for a downstream task in an end-to-end fashion. We 
empirically demonstrate the relevance of our model in terms of performances
compared to several state-of-the-art architectures. Remarkably, we show
that a simple GNN model leveraging on our new layer can surpass state-of-the-art performances by a relatively large margin. Finally, we also provide some
illustrative interpretations of our method and a sensitivity analysis of our model parameters.
	\label{sec:intro}
	
	\section{Fused Gromov-Wasserstein template based layer}

In order to describe our novel template-based GNN layer we first introduce more formally the FGW distance and its properties. In the following we denote by $\Sigma_n := \{ \vh \in \sR_+^n | \sum_{i} h_i = 1 \}$ the probability simplex
with $n$-bins, and by $\sS_n(\sA)$ the set of symmetric matrices of size $n$
taking values in $\sA \subset \R$.

\subsection{Fused-Gromov Wasserstein distance \label{subsec:FGW}} 

An undirected attributed graph $\mathcal{G}$ with $n$ nodes can be modeled in
the OT context as a tuple $(\mC, \mF, \vh)$, where $\mC \in \sS_n(\sR)$ is a
matrix encoding relationships between nodes, $\mF = (\vf_1,..., \vf_n)^\top \in
\R^{n \times d}$ is a node feature matrix and $\vh \in \Sigma_n$ is a vector of
weights modeling the relative importance of the nodes within the graph (Figure
\ref{fig:schema}, left). Without any prior knowledge, uniform weights can be chosen
($\vh = \bm{1}_{n} / n$). The matrix $\mC$ can be the graph adjacency matrix,
the shortest-path matrix or any other description of the node relationships
(i.e. the topology) of the graph \cite{peyre2016gromov, vayer2020fused,
chowdhury2021generalized}. Let us now consider two such graphs $(\mC, \mF, \vh)$
and $(\overline{\mC}, \overline{\mF}, \overline{\vh})$, of respective sizes $n$
and $\overline{n}$ (with possibly $n \neq \overline{n}$). The Fused
Gromov-Wasserstein ($\FGW$) distance is defined for $\alpha \in [0, 1]$ as
\cite{vayer2020fused, titouan2019optimal}:
\begin{equation} \label{eq:FGW}
	\FGW_{\alpha}(\mC, \mF, \vh, \overline{\mC}, \overline{\mF}, \overline{\vh}) = \min_{\mT \in \mathcal{U}(\vh, \overline{\vh})} \sum_{ijkl} \left( \alpha (C_{ij} - \overline{C}_{kl})^2 + (1 -\alpha) \| \vf_i - \overline{\vf}_k \|_2^2 \right) T_{ik} T_{jl}
\end{equation}
where $\mathcal{U}(\vh, \overline{\vh}) := \{\mT \in \R_+^{n \times
\overline{n}} | \mT \bm{1}_{\overline{n}} = \vh, \mT^\top \bm{1}_n =
\overline{\vh}\}$ is the set of admissible coupling between $\vh$ and
$\overline{\vh}$. $\FGW$ aims at finding an optimal coupling $\mT^{\star}$ by minimizing a trade-off cost, via $\alpha$, between a Wasserstein (W)
cost on the features and a Gromov-Wasserstein (GW) cost on the similarity matrices, both sharing the same coupling. The
optimal coupling $\mT^{\star}$ acts as a soft matching of the nodes, which tends to
associate pairs of nodes that have similar pairwise relations in $\mC$ and
$\overline{\mC}$ (GW cost), and similar features in $\mF$ and $\overline{\mF}$
(W cost). 

Interestingly,  $\FGW$ defines a metric on the space of attributed graphs. In particular, if $\mC$
and $\overline{\mC}$ are shortest-path matrices, the $\FGW$ distance vanishes if and only if the two attributed graphs are the same up to a permutation \cite[Theorem 3.2]{titouan2019optimal}. 
Such invariance involves that two graphs \emph{strongly} isomorphic according to
Weisfeiler-Lehman base tests \cite{Leman2018THERO, Togninalli19} will have a zero $\FGW$ distance for any $\alpha$ and, more importantly, $\FGW_{\alpha} = 0$ implies that the graphs are strongly isomorphic\footnote{{Two graphs $(\mC, \mF, \vh)$ and $(\overline{\mC}, \overline{\mF}, \overline{\vh})$ are strongly isomoprhic if $n = \overline{n}$ and there exists a permutation matrix $\mathbf{P} \in \{0,1\}^{n \times n}$ such that $\overline{\mC} = \mathbf{P} \mC \mathbf{P}^{\top}, \overline{\mF} = \mathbf{P} \mF$ and $\overline{\vh} = \mathbf{P} \vh$}}.
When $\mC, \overline{\mC}$ are any symmetric matrices, we can mention
that GW ($\alpha = 1$) also defines a pseudo-distance \cite[Theorem
5.8]{sturm2012space} with respect to the notion of \emph{weak} isomorphism \cite{sturm2012space, chowdhury2021generalized}. 


\paragraph{Solving for FGW.}
The optimization problem \ref{eq:FGW} is a
non-convex quadratic program \cite[equation 6]{titouan2019optimal}, whose
non-convexity comes from the GW cost. A possible optimization procedure to solve this
problem is a Conditional Gradient (CG) algorithm, which is known to converge to a
local optimum \cite{lacoste-julien-convergence-2016}. The computational complexity of each iteration is $O(n^2 \overline{n} + \overline{n}^2n)$ \cite{peyre2016gromov}. Thus, if two graphs of considerably different sizes are considered, the complexity is quadratic  with respect to the largest size. Existing attempts to reduce
this computational cost either exploit entropic regularization of OT
\cite{peyre2016gromov,scetbon2021lineartime} or graph partitioning \cite{xu2019scalable,chowdhury2021quantized}.

\subsection{Template-based (T)FGW Graph Neural Networks \label{subsec:TFGW}} 

Building upon the $\FGW$ distance and its properties, we propose a simple layer for a GNN that takes a graph $\left(\mC, \mF, \vh \right)$ as
input and computes its $\FGW$ distances to a list of $K$ \emph{template graphs}
$\overline{\mathcal{G}}:=\{(\overline{\mC}_k, \overline{\mF}_k, \overline{\vh}_k)\}_{k \in \integ{K}}$
as follows :
\begin{equation}
	\TFGW_{\overline{\mathcal{G}},\alpha}\left(\mC, \mF, \vh \right) := \left[ \FGW_{\alpha}(\mC, \mF, \vh,  \overline{\mC}_k, \overline{\mF}_k, \overline{\vh}_k) \right]_{k=1}^K 
	\label{eq:tfgw}
\end{equation}
We postulate that this graph representation can be discriminant between the observed graphs due to $\FGW$. This claim relies on the theory of \cite{balcan2008theory} allowing one to learn provably strongly
discriminant classifiers based on the distances from the observed graphs and templates that are sampled from the dataset 
(see e.g. \cite{rakotomamonjy2018distance}  adopting the Wasserstein distance). However such an approach often requires a large
amount of templates which might be prohibitive if the distance is costly to
compute. Instead, we propose to {\bf learn} the graph templates $\overline{\mathcal{G}}$ in a supervised manner. In the same way, we also learn the trade-off parameter $\alpha$ on the data. As such, the $\TFGW$ layer can
automatically adapt to the data whose
discriminating information can be discovered either in the features or in the structure of the graphs, or in a combination of the two. Moreover, the template structures can leverage on any type of input representation $\mC_i$ since they are learnt directly from the data. 
Indeed, in the numerical experiments we implemented the model using either adjacency matrices (ADJ) that provide more interpretable templates (component $C_{i,j}\in[0,1]$ can be seen as a probability of link between nodes) or shortest path matrices (SP) that are more complex to interpret but encode global relations between the nodes.

The $\TFGW$ layer can be used directly as a first layer to build a graph representation feeding a fully connected network (MLP) 
for \textit{e.g.} graphs classication. In order to enhance the discriminating power of the model, we propose to put a GNN (denoted by $\phi_{\vu}$ and parametrized by $\vu$) on top of the $\TFGW$ layer. We assume in the remainder that this GNN model $\phi_{\vu}$ is injective in order to preserve isomorphism relations between graphs (see \cite{xu2018powerful} for more details). With a slight abuse of notation, we write $\phi_{\vu}(\mathbf{F})$ to denote the feature matrix of an observed graph after being processed by the GNN.



\paragraph{Learning with TFGW-GNN.} We focus on a classification
task where we observe a dataset $\mathcal{D}$ of $I$ graphs
$\{\mathcal{G}_i = ( \mC_i, \mF_i, \vh_i)\}_{i \in \integ{I}}$ with variable
number of nodes $\{n_i\}_{i \in \integ{I}}$ and where each graph is assigned to a label $y_i \in \mathcal{Y}$, with
$\mathcal{Y}$ a finite set. The full model is illustrated in
Figure \ref{fig:schema}. We first process the features of the nodes of the input graphs via the GNN
$\phi_{\vu}$, then use the $\TFGW$ layer to represent the graphs as
vectors in $\mathbb{R}^K$. Finally we use the final MLP model $\psi_{\vv} : \mathbb{R}^K
\rightarrow \mathcal{Y}$ parameterized by $\vv$, to predict the label for any input graph. 
The whole model is learned in a end-to-end fashion by minimizing the cross-entropy loss on the whole
dataset leading to the following optimization problem :
\begin{equation}\label{eq:model}
	\min_{ \vu, \vv, \{ (\overline{\mC}_k, \overline{\mF}_k, \overline{\vh}_k)\},\alpha} \quad\frac{1}{I}\sum_{i =1}^I \mathcal{L}\left(y_i, \psi_{\vv}\left(\TFGW_{\overline{\mathcal{G}},\alpha}\left(\mC_i, \phi_{\vu}(\mF_i), \vh_i \right) \right) \right).
\end{equation}

Notable parameters of \eqref{eq:model} are the template graphs in the embeddings $ \{
(\overline{\mC}_k, \overline{\mF}_k, \overline{\vh}_k)\}$ and more precisely
their pairwise node relationship $\overline{\mC}_k$, node features
$\overline{\mF}_k$ and the distribution on the nodes
$\overline{\vh}_k\in\Delta_K$ on the simplex. The last parameter
reweighs individual nodes in each template and performs nodes selection
when some weights are exactly $0$ \cite{vincent2021online, vincent-cuaz2022semirelaxed}. Finally, the global parameter $\alpha$
is also learnt from the whole dataset.
Although it is possible to learn a different $\alpha$ per template, we observed that this extra level of flexibility is prone to overfitting, and we will not consider it in the experimental section.


\paragraph{Optimization and differentiation of TFGW.} 
We propose to solve the optimization problem in \eqref{eq:model} using stochastic
gradient descent. The $\FGW$ distances are
computed by adapting the conditional gradient solver implemented in the POT
toolbox \cite{flamary2021pot}.
The solver was designed to allow backward propagation of the gradients \textit{w.r.t.}
all the parameters of the distance and was adapted to also compute the gradient
\emph{w.r.t.} the parameter $\alpha$. The gradients are obtained using the
Envelop Theorem \cite{afriat1971theory} allowing to keep $\mT^\star$ constant.
We used Pytorch~\cite{paszke2017automatic} to implement the model. 
The template structure $\overline{\mC}_k$, node weights $\overline{\vh}_k$ and $\alpha$ are updated with a projected gradient respectively on the set of 
symmetric matrices $\sS_{\overline{n}_k}(\R_+)$ ($\sS_{\overline{n}_k}([0,1])$
when $\mC_i$ are adjacency matrices), the simplex $\Delta_K$ and $[0,1]$.
The projection onto the probability simplex of the node weights leads to sparse solutions \cite{condat2016fast}, therefore
the size of each $(\overline{\mC}_k, \overline{\mF}_k, \overline{\vh}_k)$ can decrease along iterations hence reducing the \emph{effective} 
number of their parameters to optimize. This way the numerical solver can leverage on the fact that many computations are unnecessary as soon as the weights are set to zero. Note that the FGW solver from POT uses an OT
solver implemented in C++ on CPU which means that it comes with some overhead (memory
transfer between GPU and CPU) when training the model on GPU. Still the multiple FGW distances computation has been
implemented in parallel on CPU with a computational time that remains reasonable in
practice (see experimental section \ref{sec:experiments}). While a GPU solver can be found when using entropy regularized FGW, it introduces a new parameter related to the regularization strength which is more cumbersome to set, and that we did not consider it in the experiments.


\paragraph{Properties of the TFGW layer.} We now discuss a property of the proposed layer 
resulting from the properties of $\FGW$ (see Section \ref{subsec:FGW}). We have the following result:
\begin{lemma}
	The $\TFGW$ embeddings are invariant to strong isomorphism.
\end{lemma}
This lemma directly stems from the fact that $\FGW$ is invariant to strong isomorphism of one of its inputs. This proposition implies that two graphs with any aforementioned representation which only differ by a permutation of the nodes will share the same $\TFGW$ embedding. 
Moreover such a property holds for any mapping $\phi_u$ which is
injective, such as a Multi-Layer Perceptron (MLP)
\cite{hornik1989multilayer} or any GNN with a sum aggregation scheme as described in \cite{xu2018powerful}. 

Moreover, the optimal coupling $\mT^{\star}$ resulting from $\TFGW$ between $(\mC_i,\phi_{\vu}(\mF_i), \vh_i)$ and the template $(\overline{\mC}_k,\overline{\mF}_k, \overline{\vh}_k)$, will encode correspondances between the nodes of the graph and the nodes of the template that will be propagated during the backward operation. The size of the inputs, the size of the templates and their respective weight $\overline{\vh}_k$ will play a crucial role regarding this operation.
Also note that, since the templates are estimated here to optimize a supervised task, they will promote discriminant distance embedding instead of graph reconstruction quality as proposed in other $\FGW$ unsupervised learning methods
\cite{vincent2021online, vincent-cuaz2022semirelaxed}.  

	\label{sec:model}
	
	\section{Numerical experiments}
	

This section aims at illustrating the performances of our approach for graph classification in synthetic and real-world datasets. First, we showcase the relevance of our $\TFGW$ layer on existing synthetic datasets known to require expressiveness beyond the WL-test (Section \ref{sec:beyondWL}). Then we benchmark our model with state-of-the-art approaches on well-known real-world datasets (Section \ref{sec:classif}). We finally discuss our results through a sensitivity analysis of our models (Section \ref{sec:sensiti}).

\subsection{Synthetic datasets beyong WL test \label{sec:beyondWL}}
\begin{wraptable}{r}{0.4\textwidth} 
	\vspace{-5mm}
	\caption{Average accuracy on synthetic datasets (10 simulations). } \label{tab:toy}
	\begin{center}
		\scalebox{0.8}{
			\begin{tabular}{|c|c|c|c|c|}
				\hline
				model & {4-CYCLES}& {SKIP-CIRCLES} \\ \hline
				TFGW &  \textbf{0.99(0.03)} &  \textbf{1.00(0.00)} \\
				TFGW-fix &  0.63(0.11) &  \textbf{1.00(0.00)}\\ \hline
				GIN &  0.50(0.00)&  0.10(0.00)\\
				DropGIN &   \textbf{1.00(0.01)} &  0.82(0.28)\\ \hline
		\end{tabular}}
	\end{center}
\end{wraptable}

Identification of graphs beyond the WL test is one important
challenge faced by the GNN community. In order to test the ability of $\TFGW$ to handle
such fundamentally difficult problems we consider two synthetic datasets:
4-CYCLES \cite{loukas2019graph, papp2021dropgnn} contains graphs with
(possibly) disconnected cycles where the label $y_i$ is the presence of a cycle of
length 4; SKIP-CIRCLES \cite{chen2019equivalence} contains circular graphs with
skip links and the labels (10 classes) are the lengths of the skip links among $\{2, 3, 4, 5, 6,
9, 11, 12, 13, 16\}$. 

We compare the performances of the $\TFGW$ layer for embedding such graphs with GIN \cite{xu2018powerful} designed to be at least as expressive
as the WL test, and DropGIN \cite{papp2021dropgnn} which proposed a successful dropout technique to overcome some drawbacks of GINs.
We replicate the benchmark
of \cite{papp2021dropgnn} by considering 
for both GIN and DropGIN, 4 GIN layers for 4-CYCLES, and 9 GIN
layers for SKIP-CIRCLES as the skip links can form cycles of up to 17 hops. Since the graphs do not have features we use
directly the $\TFGW$ on the raw graph representation with $\alpha=1$ hence computing
only the GW distance.\begin{wrapfigure}{r}{0.25\textwidth}  \vspace{-5mm}
	\begin{center}
		\includegraphics[width=0.25\textwidth]{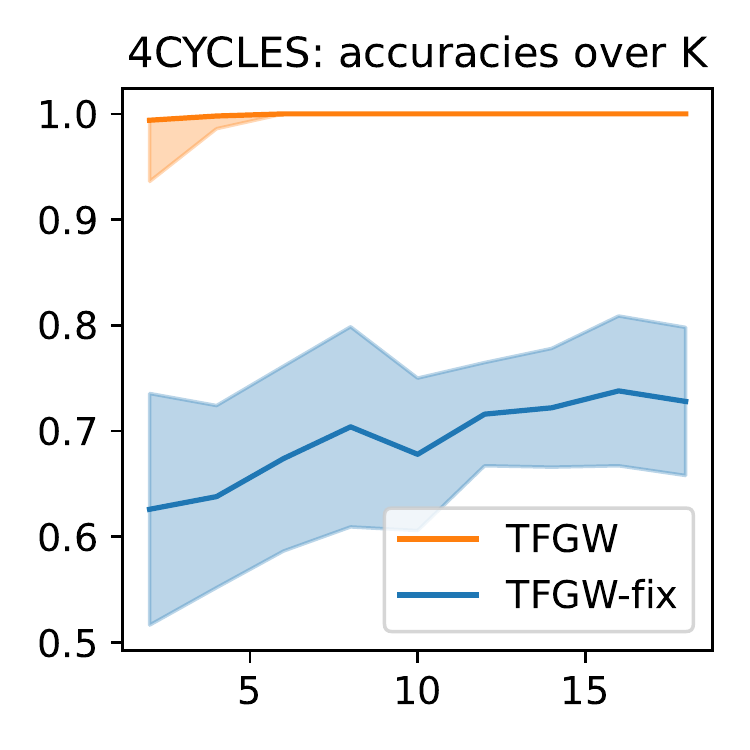}
	\end{center}
	\vspace{-4mm}
	\caption{\label{fig:toy}Test accuracy distributions by number of templates either fixed or learned.} 
\end{wrapfigure} The GNN methods above artificially adds a feature equal to $1$ on all nodes as they have the same degree. For these experiments we use adjacency
matrices for $\mC_i$ and we investigate two flavours of $\TFGW$: 1) in $\TFGW$-fix we fix the
templates by sampling \emph{one template per class} from the training dataset (this can be seen as a simpler FGW feature extraction);
2) for $\TFGW$ we learn the templates from the training
data (as many as the number of classes) as proposed in the previous sections. Results are averaged over 10 runs and reported in
Table \ref{tab:toy}. $\TFGW$ based methods perform very well on both datasets with
impressive results on SKIP-CIRCLE when GNN have limited performances. This is due to the fact that different samples from one class of SKIP-CIRCLE are generated by permuting nodes of the same graph and FGW distances are invariant to these permutations.
4-CYCLES has a more complex structure with intra-class
heterogeneity and requires more than two templates to perform as good as DropGIN. To illustrate this we have computed the accuracy on this dataset as a
function of the number of templates $K$ in Figure \ref{fig:toy}. We can see that a perfect
classification is reached up to $K=4$ for $\TFGW$, while $\TFGW$-fix still
struggles to generalize at $K=20$. This illustrates that
learning the templates is essential to keep $K$ (and numerical complexity) small
while ensuring good performances. 


\subsection{Graph classification benchmmark \label{sec:classif}}
\begin{table}[t!]
	\caption{\label{tab:benchmark}Test set classification accuracies from 10-fold CV. The first (resp. second) best performing method is highlighted in bold (resp. underlined). (*) results on the validation sets from 10-fold CV of original paper.}
	\begin{center}
		\scalebox{0.7}{
			\begin{tabular}{|c|c||c|c|c|c|c||c|c|c|}
				\hline
				category & model & MUTAG & PTC & ENZYMES & PROTEIN & NCI1& IMDB-B & IMDB-M & COLLAB\\ \hline\hline
				Ours & TFGW ADJ (L=2)  & \textbf{96.4(3.3)} & \textbf{72.4(5.7)} & \underline{73.8}(\underline{4.6}) & \textbf{82.9(2.7)}  & \textbf{88.1(2.5)} & \textbf{78.3(3.7)}  & \textbf{56.8(3.1)} &\textbf{84.3(2.6)} \\
				& TFGW SP (L=2)  & \underline{94.8}(\underline{3.5}) & \underline{70.8}(\underline{6.3})& \textbf{75.1(5.0)} & \underline{82.0}(\underline{3.0}) & \underline{86.1}(\underline{2.7}) &  \underline{74.1}(\underline{5.4})  & \underline{54.9}(\underline{3.9}) &  80.9(3.1) \\ \hline
				OT emb. & OT-GNN (L=2)& 91.6(4.6)& 68.0(7.5)& 66.9(3.8) & 76.6(4.0)& 82.9(2.1)& 67.5(3.5) & 52.1(3.0)& 80.7(2.9)\\ 
				& OT-GNN (L=4)& 92.1(3.7)& 65.4(9.6)& 67.3(4.3)& 78.0(5.1) & 83.6(2.5) & 69.1(4.4)& 51.9(2.8) & 81.1(2.5)\\
				&WEGL & 91.0(3.4)& 66.0(2.4)& 60.0(2.8)& 73.7(1.9) &75.5(1.4) & 66.4(2.1) & 50.3(1.0) &79.6(0.5) \\
				
				\hline
				GNN & PATCHYSAN & 91.6(4.6) & 58.9(3.7) & 55.9(4.5) & 75.1(3.3) & 76.9(2.3) & 62.9(3.9)& 45.9(2.5)& 73.1(2.7) \\
				& GIN & 90.1(4.4) & 63.1(3.9) & 62.2(3.6)&76.2(2.8) & 82.2(0.8) & 64.3(3.1)& 50.9(1.7)& 79.3(1.7)\\
				& DropGIN & 89.8(6.2) & 62.3(6.8) & 65.8(2.7)&76.9(4.3) & 81.9(2.5) & 66.3(4.5) &  51.6(3.2)& 80.1(2.8) \\
				& PPGN* & 90.6(8.7) & 66.2(6.5) & - & 77.2(4.7) & 83.2(1.1) & 73.0(5.8) & 50.5(3.6) & 81.4(1.4)\\
				& DIFFPOOL & 86.1(2.0) & 45.0(5.2) & 61.0(3.1) & 71.7(1.4) & 80.9(0.7) & 61.1(2.0)& 45.8(1.4) &80.8(1.6) \\ \hline
				Kernels & FGW - ADJ & 82.6(7.2) & 55.3(8.0)& 72.2(4.0) & 72.4(4.7) & 74.4(2.1)&  70.8(3.6) & 48.9(3.9)& 80.6(1.5) \\
				& FGW - SP & 84.4(7.3) & 55.5(7.0) & 70.5(6.2) & 74.3(3.3) & 72.8(1.5)& 65.0(4.7) & 47.8(3.8)& 77.8(2.4)\\
				& WL  & 87.4(5.4) & 56.0(3.9) & 69.5(3.2)& 74.4(2.6) & 85.6(1.2) &67.5(4.0) & 48.5(4.2)& 78.5(1.7)\\
				& WWL & 86.3(7.9) & 52.6(6.8) & 71.4(5.1) & 73.1(1.4) & 85.7(0.8) & 71.6(3.8) & 52.6(3.0)& \underline{81.4}(\underline{2.1})\\
				\hline \hline
				 & Gain with TFGW & \textbf{4.3} & \textbf{4.4} & \textbf{2.9} & \textbf{4.9}& \textbf{2.4}& \textbf{5.3}& \textbf{4.2} & \textbf{2.9}\\ \hline
		\end{tabular}}
	\end{center}
\end{table}

We now evaluate and compare the performances of our $\TFGW$ GNN with a
number of state-of-the-art  graph classifiers, from kernel methods to GNN. The
numerical experiments are conducted on real life graph datasets to provide a fair benchmark of all methods on several heterogeneous graph structures.

\paragraph{Datasets.} We use 8 well-known graph classification datasets \cite{KKMMN2016}: 5
bioinformatics datasets among which 3  have discrete node features (MUTAG,
PTC, NCI1 \cite{kriege2012subgraph,shervashidze2011weisfeiler}) and 2 have continuous node features (ENZYMES, PROTEINS\cite{borgwardt2005protein}) and 3
social network datasets (COLLAB, IMDB-B, IDBM-M \cite{yanardag-deep-2015}). In order to analyse them with all methods, we augment unattributed graphs from social networks with node degree features. Detailed description and statistics on these datasets are reported in the supplementary material.

\paragraph{Baselines.} We benchmark our approaches to the following
state-of-the-art baselines for graphs classification, split into 3 categories:
i) \emph{kernel based approaches}, including $\FGW$ \cite{titouan2019optimal} operating on
adjacency and shortest-path matrices, the WL subtree kernel
\cite[WL]{shervashidze2011weisfeiler} and the Wasserstein WL kernel
\cite[WWL]{Togninalli19}. For these methods that do not require a stopping criterion
dependent on a validation set, we report results using for parameter validation a 10-fold nested cross-validation
\cite{titouan2019optimal,kriege2020survey} repeated 10 times. 
ii) \emph{OT based representation learning} models, including WEGL
\cite{kolouri2020wasserstein} and OT-GNN \cite{chen2020optimal}. iii) \emph{GNN models},
with global or more sophisticated pooling operations, including PATCHY-SAN
\cite{niepert2016learning}, DIFFPOOL \cite{ying2018hierarchical}, PPGN
\cite{maron2019provably}, GIN \cite{xu2018powerful} and its augmented version
through structure perturbations DropGIN \cite{papp2021dropgnn}.
For all these methods, we adopt the hyper-parameters suggested in the respective  papers, but with a slightly different model selection scheme, as detailed in the next paragraph.


\paragraph{Benchmark settings.} Recent GNN literature \cite{maron2018invariant,
xu2018powerful,maron2019provably,papp2021dropgnn} successfully addressed many
limitations in terms of model expressiveness compared to the WL tests.
Within that scope, they suggested to benchmark their models using a 10-fold
cross-validation (CV) where the best average accuracy on the validation folds was
reported. We suggest here to quantify the generalization capacities of GNN based
models by performing a 10-fold cross validation with a holdout test set never
seen during training. For each split, we track the accuracy on the validation
fold every 5 epochs, then the model whose parameters maximize that
accuracy is retained. Finally, the model used to predict on the holdout test set is the one with maximal validation accuracy averaged across all 
folds. This setting is more realistic than a simple 10-fold
CV and allows a better understanding of the generalization performances \cite{bengio2003no}. This point explains why some existing approaches have here different performances than those reported in their original paper.

For all the $\TFGW$ based approaches we empirically study the impact of the
input structure representation by considering adjacency (ADJ) and
shortest-path (SP) matrices $\mC_i$. For all template based models, we set the size of
the templates to the median size of the observed graphs. 

We validate the number of templates $K$ in $\{\beta |\mathcal{Y}|\}_\beta$, with $\beta \in \{2,4,6,8\}$ and $|\mathcal{Y}|$ the number of classes. Only for ENZYMES with 6
classes of 100 graphs each, we validate $\beta \in \{1, 2, 3,
4\}$.
All parameters of our $\TFGW$ layers highlighted in \textcolor{myred}{red} in Figure
\ref{fig:schema} are learned while $\phi_\vu$ is a GIN architecture
\cite{xu2018powerful} composed of $L=2$ layers aggregated using the Jumping Knowledge scheme \cite{xu2018representation}
known to prevent overfitting in global pooling frameworks. For OT-GNN 
we validate the number of GIN layers in $L \in \{2,4\}$. 
 Finally for fairness, we validate the number
of hidden units within the
GNN layers and the application of dropout on the final MLP for predictions,
similarly to GIN and DropGIN.


\paragraph{Results analysis.} The results of the comparisons in terms of accuracy
are reported in Table \ref{tab:benchmark}. Our $\TFGW$ approach consistently
\emph{outperforms with significant margins} the state-of-the-art approaches from all
categories. Even if most of the benchmarked models can perfectly fit the train sets by learning implicitly the graphs structure \cite{xu2018powerful,papp2021dropgnn,maron2019provably}, enforcing such knowledge explicitly as our $\TFGW$ layer does (through $\FGW$ distances) leads to considerably stronger generalization performances. On 7 out of 8 datasets, $\TFGW$ leads to better performances while operating on adjacency matrices (TFGW ADJ) than on shortest-path ones (TFGW SP).
Interestingly, this ranking with respect to those input representations does not
necessarily match the one of the FGW kernel which extracts knowledge from the
graph structures $\mC_i$ through $\FGW$, as our $\TFGW$ layer. These different
dependencies to the provided inputs may be due to the GNN pre-processing of node
features which suggests the study of its ablation. \begin{wraptable}{r}{0.4\textwidth} 
	\vspace{-5mm}
	\caption{Number of parameters and averaged prediction time per graph. } \label{tab:runtimes}
	\begin{center}
		\scalebox{0.8}{
			\begin{tabular}{|c|cc|}
				\hline
				model       & \multicolumn{2}{c|}{PTC}                        \\ \hline
				& \multicolumn{1}{c|}{parameters} & runtimes (ms) \\ \hline
				(ours) TFGW & \multicolumn{1}{c|}{25.1k}      & 12.1          \\ \hline
				OT-GNN      & \multicolumn{1}{c|}{30.8k}      & 7.6           \\ \hline
				GIN         & \multicolumn{1}{c|}{29.9k}      & 0.19          \\ \hline
				DropGIN     & \multicolumn{1}{c|}{44.1k}      & 14.3          \\ \hline
		\end{tabular}}
	\end{center}
\end{wraptable}
Finally to complete this analysis, we report in Table \ref{tab:runtimes}
the number of parameters of best selected models across various methods, for the
dataset PTC. Our $\TFGW$ leads to better classification performances while
having comparable number of parameters than these competitors. We also reported
for these models, their averaged prediction time per graph. These measures were
taken on CPUs (Intel Core i9-9900K CPU, 3.60 GHz) in order to fairly compare the
numerical complexity of these methods, as OT solver used in $\TFGW$ and OT-GNN
are currently limited to these devices (see the detailed discussion in Section
\ref{subsec:FGW}). Although the theoretical complexity of our approach is
\emph{at most} cubic in the number of nodes, we still get in practice a fairly
good speed for classifying graphs, in comparison to the other competitive
methods.

\subsection{Ablation study, sensitivity analysis and discussions \label{sec:sensiti}}
In this section we inspect the role of some of the model parameters ($\alpha$ in $\FGW$, weights estimation
in the templates, depth of the GNN $\phi_\vu$) in terms of the classification performance. To this end, we first conduct on all datasets an
ablation study on the graph template weights
$\overline{\vh}_k$ and the number of GIN layers in $\phi_\vu$. 
Then, we take a closer look at the
estimated  trade-off
parameters $\alpha$ and provide a sensitivity analysis \emph{w.r.t.} the number of templates and the number of GIN layers.

\paragraph{Ablation Study.} Following the same procedure as in Section \ref{sec:classif}, we benchmark the following settings for
our $\TFGW$ models: for adjacency (ADJ) and
shortest path (SP) representations $\mC_i$, we learn the distance layers either directly, on
the raw data (\textit{i.e.} $L=0, \phi_{\mathbf{u}}=\operatorname{id}$), or after embedding the data with $(L=1)$ or $(L=2)$ GIN layers. For $L=0$ we
either fix the graph template weights $\overline{\vh}_k$ uniformly or learn them.
The results (test accuracy) are reported in Table \ref{tab:ablation}.
Learning the weights systematically improves the generalization capabilities of our
models of at least $1\%$ or $2\%$ for both ADJ and SP graph representations. For a given number of graph templates, the weights learning allows to better fit the specificities of the classes (e.g. varying proportion of nodes in different parts of the graphs). 
Moreover, as weights can become sparse in the simplex during training they also allow the model to have templates whose number of nodes adapts to the classification objective, while bringing computational benefits as discussed in Section \ref{subsec:TFGW}.
Those observations explain why we learnt them by default in the benchmark of the previous subsection.

Next, we see in Table \ref{tab:ablation} that using GNN layers as a
pre-processing for our $\TFGW$ layer enhances generalization powers of our models,
whose best performances are obtained for $L=2$. Interestingly, for $L=0$,
$\TFGW$ with SP
matrices outperforms $\TFGW$ with ADJ matrices, meaning that the shortest path
distance brings more discriminant information on raw data. But when
$L\geq 1$ (\textit{i.e.} when a GNN pre-processes the node features), $\TFGW$ with ADJ matrices improves the accuracy. 
An explanation could be that the GNN $\phi_{\vu}$ can somehow replicate (and  outperform) a SP metric between nodes. This emphasizes that the strength of our approach clearly exhibited in Table \ref{tab:benchmark} lies in the inductive bias of our FGW distance embedding.


\begin{table}[t!]
	\caption{\label{tab:ablation}Classification results from 10-fold cross-validation of our $\TFGW$ models in various scenarios: for $L \in \{0,1,2\}$ GIN layers, we either fix templates weights $\overline{\vh}_k$ to uniform distributions or learn them. The first and second best performing method are respectively highlighted in bold and underlined. }
	\begin{center}
		\scalebox{0.7}{
			\begin{tabular}{|c|c|c|c|c|c|c|c|c|c|c|}
				\hline
				model & inputs & $\overline{\vh}_k$ & MUTAG & PTC & ENZYMES & PROTEIN & NCI1& IMDB-B & IMDB-M & COLLAB\\ \hline
				TFGW (L=0)& ADJ & uniform &92.1(4.5) & 63.6(5.0) & 67.4(7.3) &  78.0(2.0) & 80.3(1.5) & 69.9(2.5) &49.7(4.1) & 78.7(3.1)  \\
				& ADJ &learnt & 94.2(3.0) & 64.9(4.1) & 72.1(5.5) & 78.8(2.2)& 82.1(2.5) & 71.3(4.3) & 52.3(2.5) &  80.9(2.7)\\
				&SP &uniform&  94.8(3.7)  & 66.5(6.7) & 72.7(6.9) &77.5(2.4)  & 79.6(3.7) & 68.1(4.4) & 48.3(3.6)& 78.4(3.4)\\
				& SP &learnt& \underline{95.9}(\underline{4.1}) & 67.9(5.8)& \textbf{75.1(5.6)} & 79.5(2.9)& 83.9(2.0)  &72.6(3.1) & 53.1(2.5) & 79.8(2.5)\\
				\hline
				TFGW(L=1)& ADJ & learnt & 94.8(3.1) & 68.7(5.8)& 72.7(5.1) &81.5(2.8) & 85.4(2.8) & \underline{76.3}(\underline{4.3}) & \underline{55.9}(\underline{2.4}) & \underline{82.6(1.8)} \\ 
				& SP & learnt & 95.4(3.5) & \underline{70.9}(\underline{5.5})& \underline{74.9}(\underline{4.8}) & \underline{82.1}(\underline{3.4})& 85.7(3.1) & 73.8(4.8)& 54.2(3.3) & 81.1(2.5)\\ \hline
				TFGW (L=2) &ADJ & learnt  & \textbf{96.4(3.3)} & \textbf{72.4(5.7)} & 73.8(4.6) & \textbf{82.9(2.7)}  & \textbf{88.1(2.5)} & \textbf{78.3(3.7)}  & \textbf{56.8(3.1)} & \textbf{84.3(2.6)} \\
				& SP & learnt&94.8(3.5) & 70.8(6.3)& \textbf{75.1(5.0)} & 82.0(3.0) & \underline{86.1}(\underline{2.7}) &  74.1(5.4)  & 54.9(3.9) & 80.9(3.1) \\ \hline
		\end{tabular}}
	\end{center}
\end{table}
\begin{figure}[t!]
	\begin{center}
		\includegraphics[height=4cm]{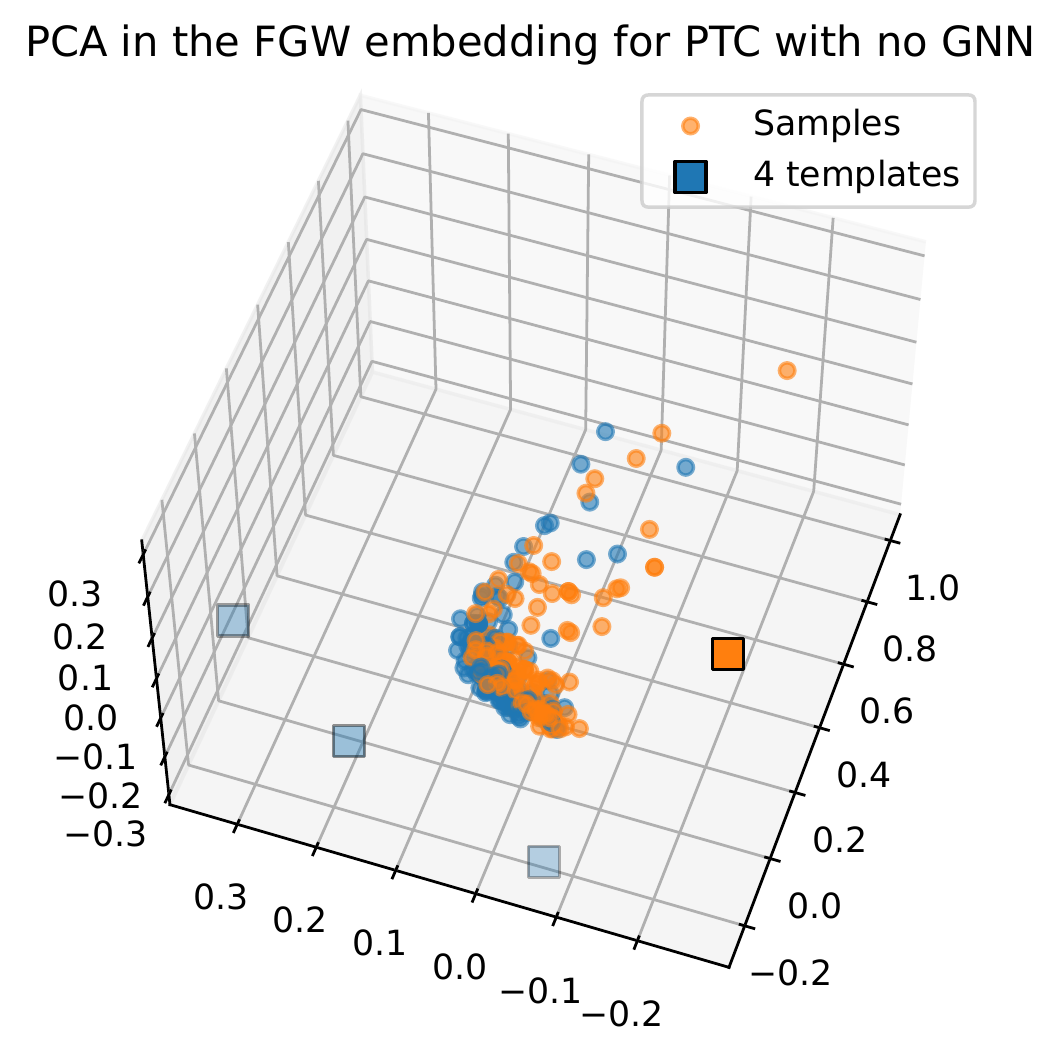}\hspace{5mm}
		\includegraphics[height=4cm]{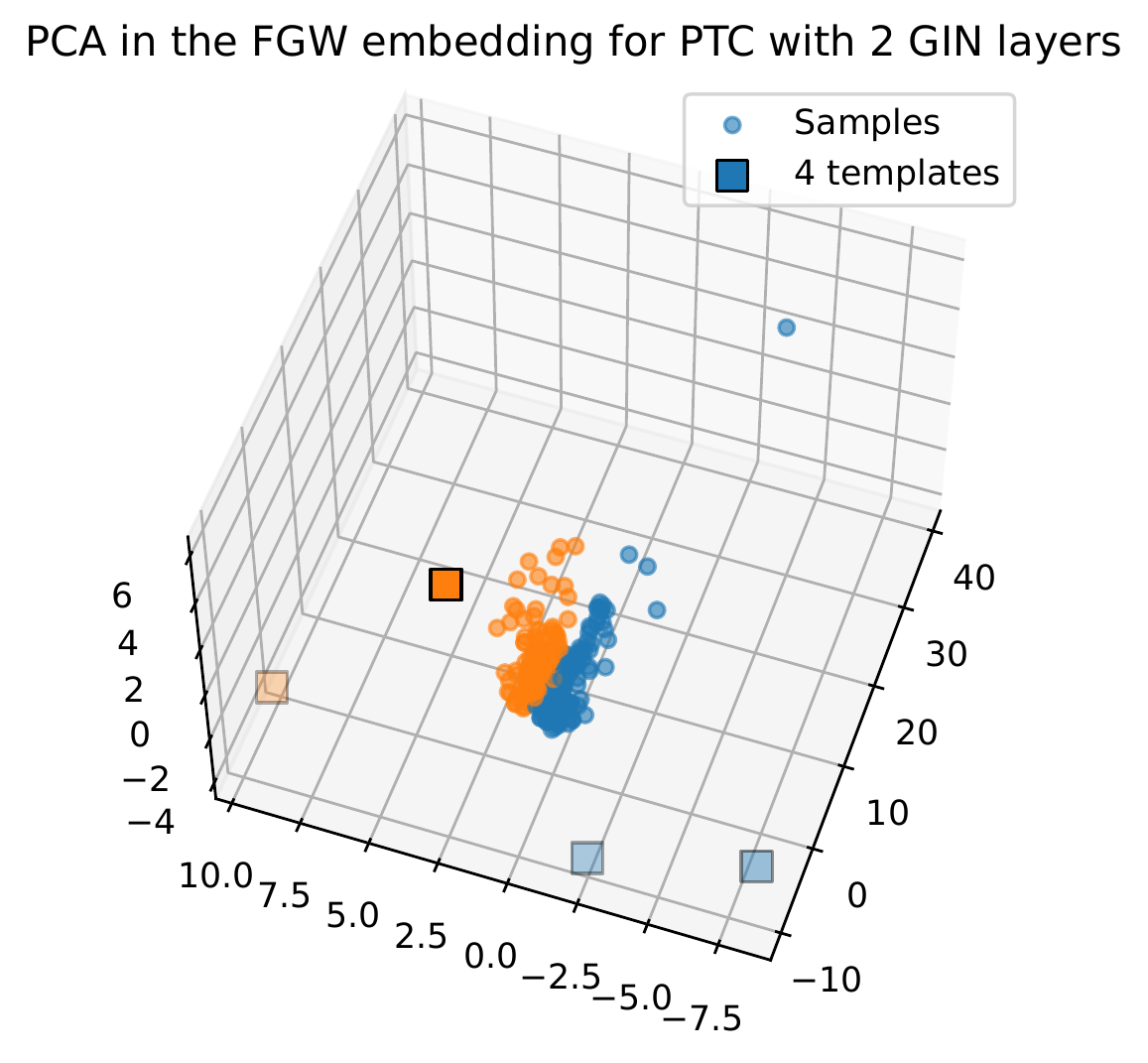}\hspace{5mm}
		\includegraphics[height=4cm]{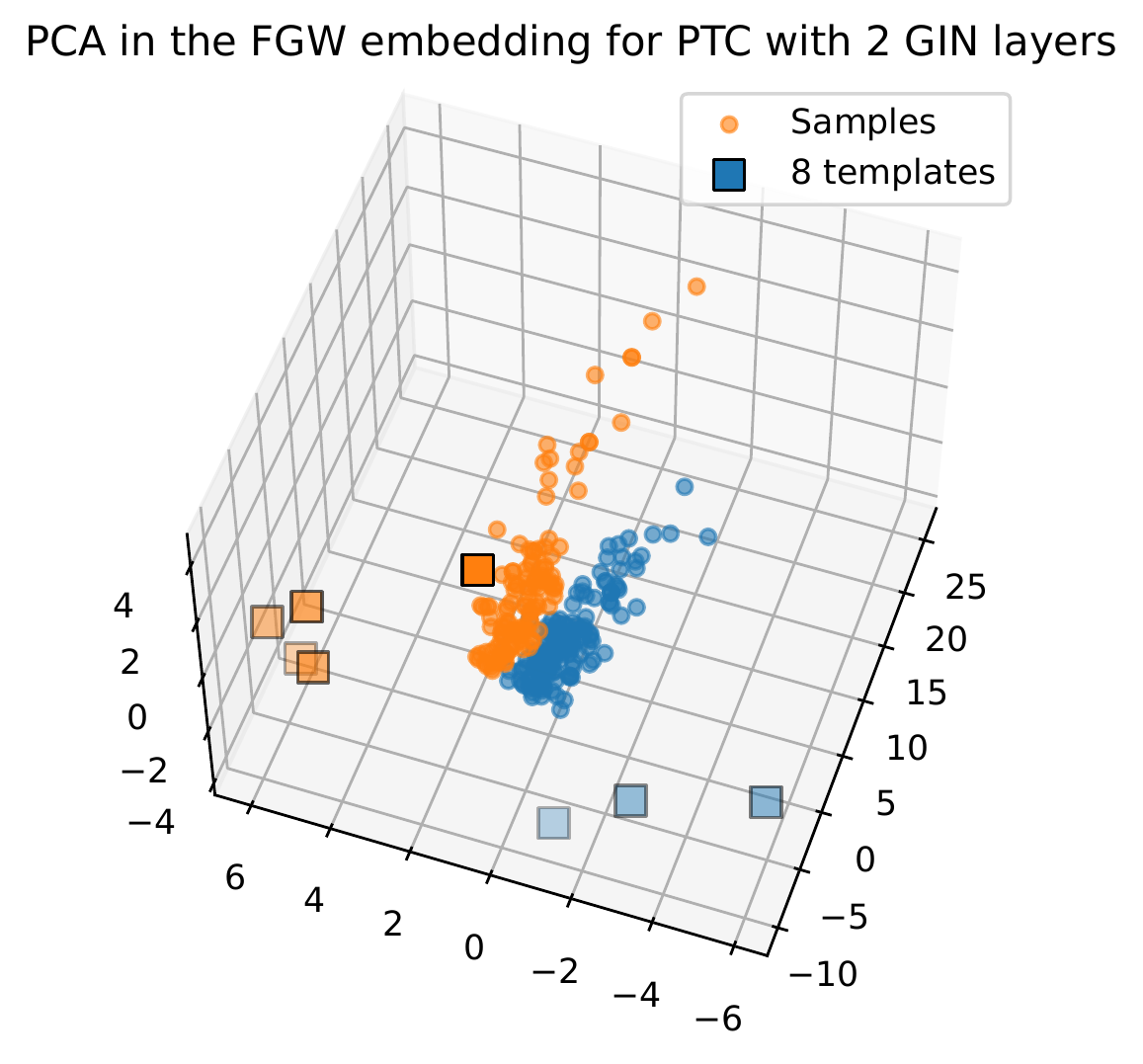}
	\end{center}\vspace{-2mm}
	\caption{PCA projections of the template based embeddings for different
		models and number of templates.(need to debug colors)\label{fig:pca} }
\end{figure}
\paragraph{Importance of the structure/feature aspect of FGW.} \begin{wrapfigure}{r}{0.5\textwidth}  \vspace{-5mm}
	\begin{center}
		\includegraphics[height=0.18\textwidth]{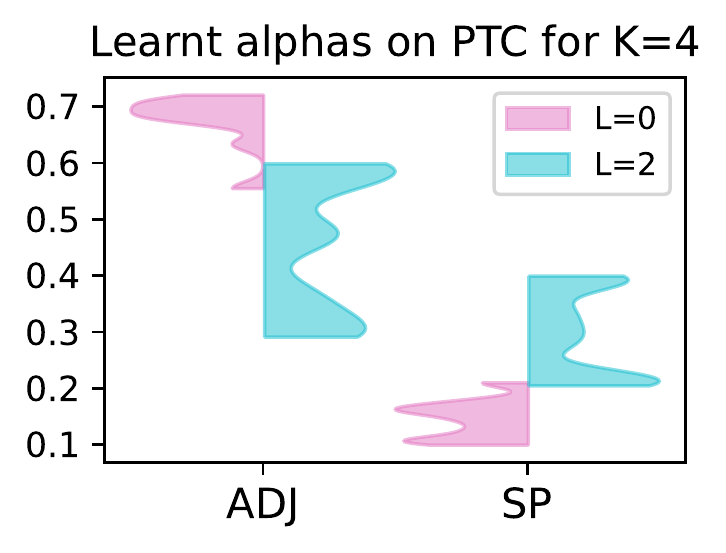}
		\includegraphics[height=0.18\textwidth]{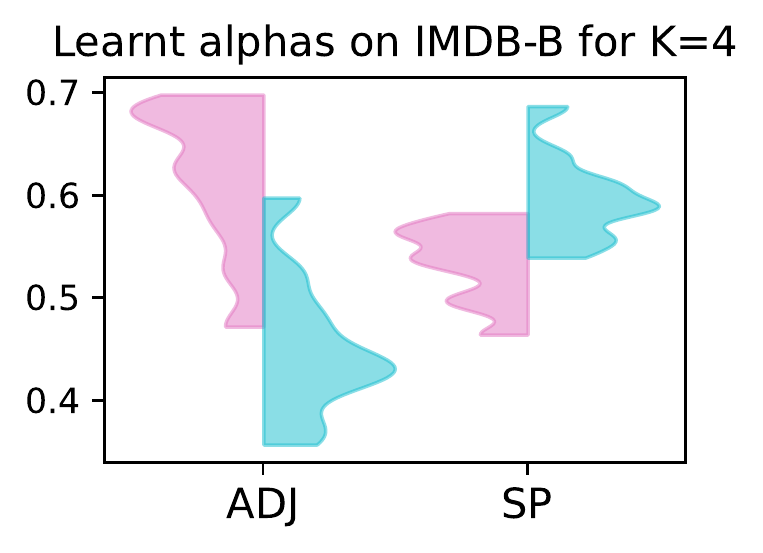}
		
	\end{center}
	\vspace{-5mm}
	\caption{\label{fig:alphas}Distributions of estimated $\alpha$.} 
\end{wrapfigure}
To the best of our knowledge we are the first to actually learn the
trade-off parameter $\alpha$ of $\FGW$ in a supervised way. For this matter, we
verify that our models did not converge to degenerated solutions where either
the structure ($\alpha =0$) or the features ($\alpha=1$ ) are omitted. To
this end we report in Figure \ref{fig:alphas} the distributions of the estimated
$\alpha$ for some models learnt on datasets PTC and IMDB-B, where features are respectively existing in the dataset or created using node degrees. We can see that for both 
kinds of input graph representations,
$\alpha$ parameters are strictly between $0$ and $1$. One can notice the variances of those
distributions illustrating the non-uniqueness of this trade-off parameter coming
from the non-convexity of our optimization problem (a given value of $\alpha$
can potentially be compensated by the scaling of the GNN output). Unfortunately, the
analysis of the real relative importance between structures and features can not
be achieved only by looking at those values as the node embeddings and templates
are different across models and data splits.

\paragraph{Visualizing the TFGW embedding.}{
In order to interpret the $\TFGW$ embedding, we illustrate in Figure
\ref{fig:pca} the PCA projection of our distance embeddings learned on
PTC with $L=0$ and $L=2$ and the number of templates $K$ varying in $\{4,8\}$. {For this experiment, we have chosen the PCA because it allows to have 
a more interpretable low dimensional projection that preserves the geometry compared to local 
neighbourhood based embeddings such as TSNE \cite{van2008visualizing} or UMAP
\cite{mcinnes2018umap}. As depicted in the figure, the learned templates are extreme
points in the embedding space of the PCA. This result is particularly interesting because existing unsupervised $\FGW$ representation learning methods
tend toward estimating templates that belong to the data manifold, or to form a
‘‘convex enveloppe'' of the data to ensure good reconstruction \cite{vincent2021online,xu2020gromov}. On the contrary, the templates learned through our approach seem to be located on a plane in the PCA space while the samples evolve orthogonally to this plane (when the FGW distance increases).
In a classification context, this means that the learned templates will not actually represent realistic graphs from the data
but might encode ‘‘exaggerated'' or ‘‘extreme'' features in order to maximize the margin
between classes in the embedding. To reinforce this intuition, we added plots of the
estimated templates in the supplementary. Finally, in the figure, the samples are coloured \textit{w.r.t.} their class and the templates are coloured by
their predicted class. Interestingly, the classes are already well separated with $4$ templates but the separation is clearly
non-linear whereas using GNN pre-processing and a larger number of templates leads to a linear separation of the two classes.}

\paragraph{Sensitivity to the number of templates and GNN layers.}
\begin{wrapfigure}{r}{0.6\textwidth}  \vspace{-6mm}
	\begin{center}
		\includegraphics[width=0.6\textwidth]{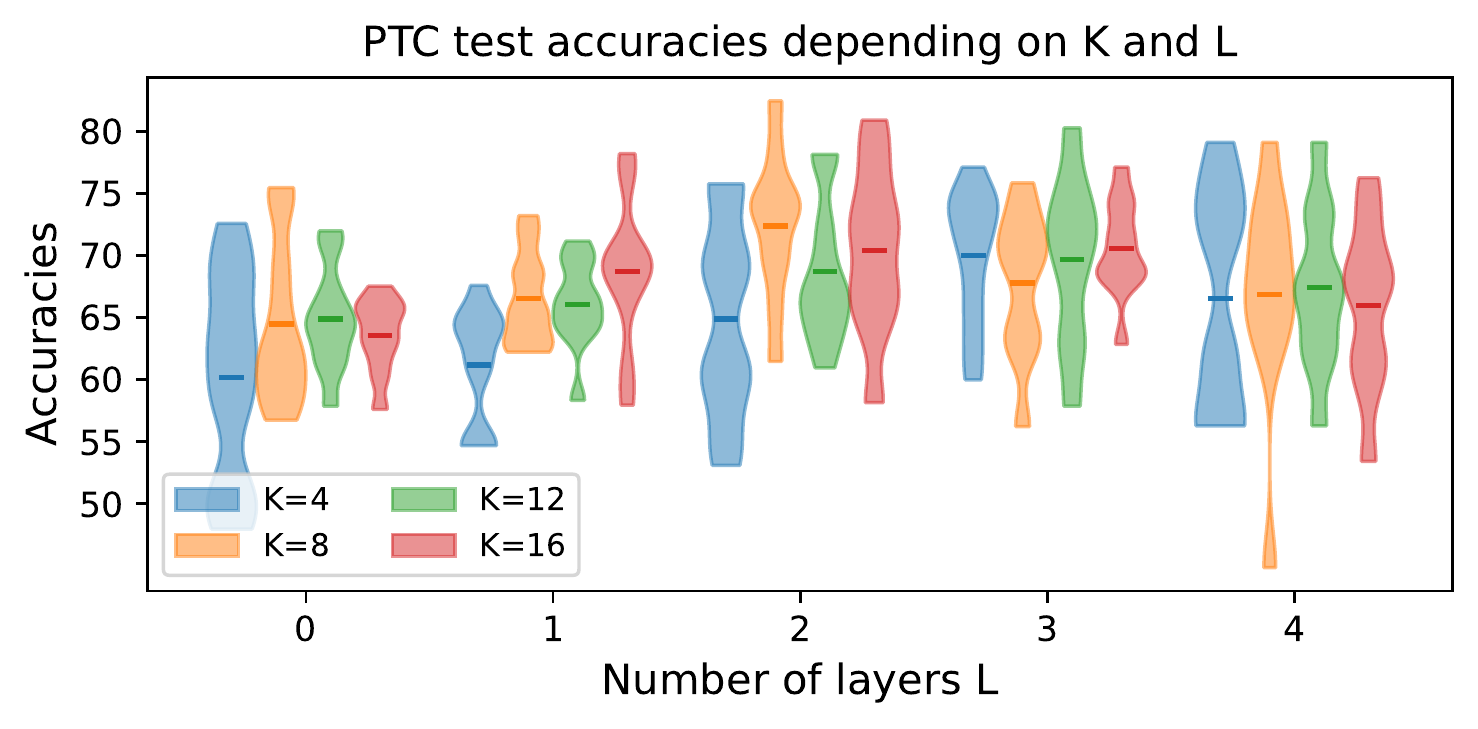}
	\end{center}
	\caption{\label{fig:KvsL}Test accuracy distributions by number of templates and number of GNN layers.\vspace{-2mm}}
\end{wrapfigure}
 To illustrate the sensitivity of our $\TFGW$ layer to the number of templates K
 and the number of GNN layers L in $\phi_{\mathbf{u}}$, we learned our models on the PTC dataset with $L$ varying in $\{0,1,2,3,4\}$. We follow the same
 procedure than in the benchmark of Section \ref{sec:classif} regarding the
 validation of $K$ and the learning process, while fixing the number of hidden
 units in the GNN layers to $16$. The test accuracy distributions for all settings
 are reported in Figure \ref{fig:KvsL}. Two phases are clearly distinguishable.
 The first one for $L \leq 2$, where for each $L$ we see that the performance
 across the number of templates steadily increases, and the second for $L>2$
 where this performance progressively decreases as a function of $L$. Moreover, in the first phase
 performances are considerably dependent on $K$ to compensate for a simple node representation, while this dependency is
 mitigated in the second which exhibits a slight overfitting. Note that these deeper
 models still lead to competitive results in comparison with benchmarked
 approaches in Table \ref{tab:benchmark}, with best averaged accuracies of
 $70.6$ $(L=3)$ and $67.4$ $(L=4)$. On one hand, these observations led us to
 set our number of layers to $L=2$ for all benchmarked datasets which lead to strong generalization power. On the other hand, deeper models might be a way to
 benefit from our $\FGW$ embeddings with very few templates which can be
 interesting from a computational perspective on larger graph datasets.
	\label{sec:experiments}
	
	\section{Conclusion}


We have introduced a new GNN layer whose goal is to represent a graph by its distances to template graphs, according to the optimal transport metric $\FGW$. The proposed layer can be used directly on raw graph data as the first layer of a GNN or can also
benefit from more involved node embedding using classical GNN layers. In a graph classification context, we combined this $\TFGW$ layer with a simple MLP model. We demonstrated on several benchmark datasets that this approach compared favorably with
state-of-the-art GNN and kernel based classifiers. A sensitivity analysis and an ablation study were presented to justify the
choice of several parameters explaining the good generalization performances.  

We believe that the new
way to represent complex structured data provided by $\TFGW$ will open the door to novel and
hopefully more interpretable GNN architectures. From a practical perspective, future works will be dedicated to combine $\TFGW$ with fast GPU solvers for network flow \cite{shekhovtsov2013distributed}. This would greatly accelerate our approach and more generally OT based deep learning methods. We also believe that the $\FGW$ distance and its existing extensions can be used with other learning
strategies including semi-relaxed $\FGW$ \cite{vincent-cuaz2022semirelaxed} for sub-graph detection.

	\label{sec:conclusion}
	
\section*{Acknowledgments}
This work is partially funded through the projects OATMIL ANR-17-CE23-0012, OTTOPIA ANR-20-CHIA-0030 
and 3IA C\^{o}te d'Azur Investments ANR-19-P3IA-0002 of the French National Research
Agency (ANR). This research was produced within the framework of Energy4Climate
Interdisciplinary Center (E4C) of IP Paris and Ecole des Ponts ParisTech. This
research was supported by 3rd Programme d'Investissements d'Avenir
ANR-18-EUR-0006-02. This action benefited from the support of the Chair
"Challenging Technology for Responsible Energy" led by l'X – Ecole polytechnique
and the Fondation de l'Ecole polytechnique, sponsored by TOTAL. This work is supported by the ACADEMICS grant of the IDEXLYON, project of the Université de Lyon, PIA operated by ANR-16-IDEX-0005.
The authors are grateful to the OPAL infrastructure from Universit\'{e} C\^{o}te d'Azur for providing resources and support.

	\bibliography{FGWM_citations}
	
	
	\section{Supplementary material}
	
\subsection{Notations}\label{subsec:notations}

An undirected attributed graph $\mathcal{G}$ with $n$ nodes can be modeled in
the OT context as a tuple $(\mC, \mF, \vh)$, where $\mC \in \sS_n(\sR)$ is a
matrix encoding relationships between nodes, $\mF = (\vf_1,..., \vf_n)^\top \in
\R^{n \times d}$ is a node feature matrix and $\vh \in \Sigma_n$ is a vector of
weights modeling the relative importance of the nodes within the graph (Figure
1 of the main paper). \emph{We always assume in the following that values in $\mC$ and $\mF$ are finite}. Let us now consider two such graphs $\mathcal{G}=(\mC, \mF, \vh)$
and $\overline{\mathcal{G}}=(\overline{\mC}, \overline{\mF}, \overline{\vh})$, of respective sizes $n$
and $\overline{n}$ (with possibly $n \neq \overline{n}$). The Fused
Gromov-Wasserstein ($\FGW$) distance is defined for $\alpha \in [0, 1]$ as
\cite{vayer2020fused, titouan2019optimal}:
\begin{equation} \label{eq:FGW_supp}
\FGW_{\alpha}(\mC, \mF, \vh, \overline{\mC}, \overline{\mF}, \overline{\vh}) = \min_{\mT \in \mathcal{U}(\vh, \overline{\vh})} \mathcal{E}^{FGW}_{\alpha}(\mC, \mF, \overline{\mC}, \overline{\mF}, \mT)  
\end{equation}
with $\mathcal{U}(\vh, \overline{\vh}) := \{\mT \in \R_+^{n \times
	\overline{n}} | \mT \bm{1}_{\overline{n}} = \vh, \mT^\top \bm{1}_n =
\overline{\vh}\}$, the set of admissible coupling between $\vh$ and
$\overline{\vh}$. For any $\mT \in \mathcal{U}(\vh, \overline{\vh})$, the FGW cost $\mathcal{E}^{FGW}_\alpha$ can be decomposed as 
\begin{equation}\label{eq:FGWcost_supp1}
\mathcal{E}^{FGW}_{\alpha}(\mC, \mF, \overline{\mC}, \overline{\mF}, \mT)  =	\alpha \mathcal{E}^{GW}(\mC, \overline{\mC}, \mT) + (1-\alpha) \mathcal{E}^{W}(\mF, \overline{\mF},\mT)
\end{equation}
which respectively refers to a Gromov-Wasserstein matching cost $\mathcal{E}^{GW}$ between graph structures $\mC$ and $\overline{\mC}$ reading as
\begin{equation}\label{eq:GWcost_supp1}
\mathcal{E}^{GW}(\mC, \overline{\mC}, \mT) = \sum_{ijkl} (C_{ij} - \overline{C}_{kl})^2  T_{ik} T_{jl}
\end{equation}
and a Wasserstein matching cost $\mathcal{E}^{W}$ between nodes features $\mF$ and $\overline{\mF}$,
\begin{equation}\label{eq:Wcost_supp1}
\mathcal{E}^{W}(\mF, \overline{\mF}, \mT) = \sum_{ik} \| \vf_i -\overline{\vf}_k \|_2^2 T_{ik}
\end{equation}
\subsection{Theoretical results}

\paragraph{Preliminaries.} 
Given two graphs $\mathcal{G}$ and $\overline{\mathcal{G}}$, we first provide a reformulation of each matching costs $\mathcal{E}^{GW}$ and $\mathcal{E}^{W}$ through matrix operations which will facilitate the readability of our proof. 

By first expanding the GW matching cost given in \ref{eq:GWcost_supp1} and using the marginal constraints over $\mT \in \mathcal{U}(\vh, \overline{\vh})$, $\mathcal{E}^{GW}$ can be expressed as 
\begin{equation}\label{eq:GWcost_supp2}
\begin{split}
\mathcal{E}^{GW}(\mC, \overline{\mC}, \mT) &= \sum_{ij}C_{ij}^2 h_i h_j + \sum_{kl} \overline{C}_{kl}^2 \overline{h}_k \overline{h}_l - 2 \sum_{ijkl} C_{ij}\overline{C}_{kl}T_{ik}T_{jl}  \\
&= \scalar{\mC^2}{\vh\vh^\top} + \scalar{\overline{\mC}^2}{\overline{\vh}\overline{\vh}^\top} - 2 \scalar{\mT^\top \mC \mT}{\overline{\mC}} \\
&= \scalar{\mT^\top \mC^2 \mT}{\bm{1}_{\overline{n} \times \overline{n}}  } + \scalar{\mT \overline{\mC}^2 \mT^\top}{\bm{1}_{n \times n} } - 2 \scalar{\mT^\top \mC \mT}{\overline{\mC}} \}\\
\\ 
\end{split}
\end{equation}
where power operations are applied element-wise and $\bm{1}^{p \times q }$ is the matrix of ones of size $p \times q$ for any integers $p$ and $q$. 

Then through similar operations $\mathcal{E}^{W}$ can be expressed as
\begin{equation}\label{eq:Wcost_supp2}
\begin{split}
\mathcal{E}_{\alpha}(\mC, \mF, \vh, \overline{\mC}, \overline{\mF}, \overline{\vh}, \mT)  & =   \sum_{i} \| \vf_i \|_2^2 h_i  + \sum_k \| \overline{\vf}_k \|_2^2 \overline{h}_k- 2 \sum_{ik}\scalar{\vf_i}{\overline{\vf}_k} T_{ik}   \\
&= \scalar{\mF^2 \bm{1}_d}{\vh} + \scalar{\overline{\mF}^2 \bm{1}_d}{\overline{\vh}} - 2 \scalar{\mF \overline{\mF}^\top}{\mT}\\
&=\scalar{\mT^\top \mF^2 }{ \bm{1}_{\overline{n} \times d}} + \scalar{\mT \overline{\mF}^2 }{ \bm{1}_{n \times d}} - 2 \scalar{\mF^\top \mT}{ \overline{\mF}^\top}\\
\end{split}
\end{equation}
\setcounter{lemma}{0}
\begin{lemma}
	The TFGW embeddings are invariant to strong isomorphism.
\end{lemma}
\begin{proof}

First, as our TFGW embeddings can operate after embedding the nodes feature of any graph, let us also introduce such an application. Given any feature matrix $\mF=(\vf_1, ..., \vf_n)^\top \subset \R^{n \times d}$, we denote by $\phi : \R^{n \times d} \rightarrow \R^{n \times d^\prime}$ an application such that $\phi(\mF) = (\varphi(\vf_1), ..., \varphi(\vf_n))^\top$ with $\varphi: \R^{d} \rightarrow \R^{d^\prime}$.

Let us now consider any pair of graphs $\mathcal{G}_1=(\mC_1, \mF_1, \vh_1)$ and $\mathcal{G}_2=(\mC_2, \mF_2, \vh_2)$ defined as in the subsection \ref{subsec:notations}. Assume that $\mathcal{G}_1$ and $\mathcal{G}_2$ are \emph{strongly isomorphic}. This is equivalent to assuming that they have the same number of nodes $n$ and there exists a permutation matrix $\mP \in \{0,1\}^{n \times n}$ such that $\mC_2 = \mP \mC_1 \mP^\top$, $\mF_2 = \mP \mF_1$ and $\vh_2 = \mP \vh_1$ \cite{titouan2019optimal,chowdhury-gromov-wasserstein-2019}.


First observe that the application $\phi$ preserves the relation of strong isomorphism. Indeed, as $\phi$ operates on each node independently through $\varphi$, we have $\phi(\mF_2)= \mP \phi(\mF_1)$ \emph{i.e},
\begin{equation}
\phi(\mF_2) = (\varphi(\mF_{2, 1}), ..., \varphi(\mF_{2, n})) = \mP(\varphi(\mF_{1, 1}), ..., \varphi(\mF_{1, n})) = \mP \phi(\mF_1)
\end{equation}
Therefore the embedded graphs $(\mC_1, \phi(\mF_1), \vh_1)$ and $(\mC_2, \phi(\mF_2), \vh_2)$ are also strongly isomorphic and are associated by the same permutation $\mP$ linking $\mathcal{G}_1$ and $\mathcal{G}_2$.

Let us consider any graph template $\overline{\mathcal{G}}= (\overline{\mC}, \overline{\mF}, \overline{\vh})$. We will prove now that the FGW cost from $(\mC_1, \phi(\mF_1), \vh_1)$ to $\overline{\mathcal{G}}$ applied in $\mT$ is the same than the FGW cost from $(\mC_2, \phi(\mF_2), \vh_2)$ to $\overline{\mathcal{G}}$ applied in $\mP\mT$. To this end we will prove that analog relations hold for the Gromov-Wasserstein and the Wasserstein matching costs independently (in this generic scenario), then we will conclude thanks the equation \ref{eq:FGWcost_supp1} which expresses FGW as a linear combination between both aforementioned costs.

First using the reformulation of $\mathcal{E}^{GW}$ of equation \ref{eq:GWcost_supp2}, we have
\begin{equation}\label{eq:invariant_GW}
\begin{split}
\mathcal{E}^{GW}(\mC_1, \overline{\mC}, \mT) &= \scalar{\mT^\top \mC^2_1 \mT}{\bm{1}_{\overline{n} \times \overline{n}}  } + \scalar{\mT \overline{\mC}^2 \mT^\top}{\bm{1}_{n \times n} } - 2 \scalar{\mT^\top \mC_1 \mT}{\overline{\mC}} \\
&=\scalar{\mT^\top \mP^\top \mC^2_2 \mP \mT}{\bm{1}_{\overline{n} \times \overline{n}}  } + \scalar{ \mT \overline{\mC}^2 \mT^\top}{ \mP^\top \bm{1}_{n \times n} \mP } - 2 \scalar{\mT^\top \mP^\top \mC_2 \mP \mT }{\overline{\mC}} \\
&=\scalar{(\mP\mT)^\top \mC^2_2 \mP \mT}{\bm{1}_{\overline{n} \times \overline{n}}  } + \scalar{ \mP\mT \overline{\mC}^2 (\mP\mT)^\top}{ \bm{1}_{n \times n} } - 2 \scalar{(\mP \mT)^\top \mC_2 \mP \mT }{\overline{\mC}} \\
&=	\mathcal{E}^{GW}(\mC_2, \overline{\mC}, \mP\mT)
\end{split}
\end{equation}

where we used $\mC_1^2 = (\mP^\top\mC_2\mP)^2 = \mP^\top\mC_2^2\mP$ and the invariance to permutations of $\bm{1}_{n \times n}$.

 Then, for $\mathcal{E}^W$ similar operations using equation \ref{eq:Wcost_supp2} and $\phi(\mF_2)^2 = (\mP \phi(\mF_1))^2 = \mP \phi(\mF_1)^2$  lead to,
\begin{equation}\label{eq:invariant_W}
\begin{split}
\mathcal{E}^W(\phi(\mF_1), \overline{\mF}, \mT)&= \scalar{\mT^\top \phi(\mF_1)^2 }{ \bm{1}_{\overline{n} \times d}} + \scalar{\mT \overline{\mF}^2 }{ \bm{1}_{n \times d}} - 2 \scalar{\phi(\mF_1)^\top \mT}{ \overline{\mF}^\top}\\
&= \scalar{\mT^\top \mP^\top \phi(\mF_2)^2 }{ \bm{1}_{\overline{n} \times d}} + \scalar{\mT\overline{\mF}^2 }{ \mP \bm{1}_{n \times d}} - 2 \scalar{\phi(\mF_2)^\top \mP \mT}{ \overline{\mF}^\top}\\
&= \scalar{(\mP\mT)^\top  \phi(\mF_2)^2 }{ \bm{1}_{\overline{n} \times d}} + \scalar{\mP\mT\overline{\mF}^2 }{ \bm{1}_{n \times d}} - 2 \scalar{\phi(\mF_2)^\top \mP \mT}{ \overline{\mF}^\top}\\
& = \mathcal{E}^W(\phi(\mF_2), \overline{\mF}, \mP\mT)\\
\end{split}
\end{equation}
Therefore, the same result holds for $FGW$ using \ref{eq:FGWcost_supp1} and equations \ref{eq:invariant_GW} \ref{eq:invariant_W} as
\begin{equation} \label{eq:invariant_FGW}
\begin{split}
\mathcal{E}^{FGW}_\alpha(\mC_1, \phi(\mF_1), \overline{\mC}, \overline{\mF}, \mT) &= \alpha \mathcal{E}^{GW}(\mC_1,\overline{\mC},  \mT) + (1 -\alpha)\mathcal{E}^{W}(\phi(\mF_1),\overline{\mF},  \mT) \\
&= \alpha \mathcal{E}^{GW}(\mC_2,\overline{\mC},  \mP\mT) + (1 -\alpha)\mathcal{E}^{W}(\phi(\mF_2),\overline{\mF},  \mP\mT) \\
&=\mathcal{E}^{FGW}_\alpha(\mC_2, \phi(\mF_2), \overline{\mC}, \overline{\mF}, \mP\mT)
\end{split}
\end{equation}
Following an analog derivation than above, one can easily prove for $\mT \in \mathcal{U}(\vh_2, \overline{\vh})$ that
\begin{equation} \label{eq:invariant_FGW_2}
\begin{split}
\mathcal{E}^{FGW}_\alpha(\mC_2, \phi(\mF_2), \overline{\mC}, \overline{\mF}, \mT) &= \mathcal{E}^{FGW}_\alpha(\mC_1, \phi(\mF_1), \overline{\mC}, \overline{\mF}, \mP^\top\mT)
\end{split}
\end{equation}

Using the relations \ref{eq:invariant_FGW} and \ref{eq:invariant_FGW_2}, we will now prove the following equality 
\begin{equation}
	\FGW_\alpha(\mC_1, \phi(\mF_1), \vh_1,  \overline{\mC}, \overline{\mF}, \overline{\vh}) = \FGW_\alpha(\mC_2, \phi(\mF_2), \vh_2,  \overline{\mC}, \overline{\mF}, \overline{\vh})
\end{equation}

First of all, the existence of optimal solutions for both FGW problems is ensured by the Weierstrass theorem \cite{santambrogio2015optimal}. We denote an optimal coupling $\mT^\star_1 \in \mathcal{U}(\vh_1, \overline{\vh})$ for $\FGW_\alpha(\mC_1, \phi(\mF_1), \vh_1,  \overline{\mC}, \overline{\mF}, \overline{\vh})$. Assume there exists an optimal coupling $\mT^\star_2$ for $\FGW_\alpha(\mC_2, \phi(\mF_2), \vh_2,  \overline{\mC}, \overline{\mF}, \overline{\vh})$ such that 
\begin{equation}
	\mathcal{E}^{FGW}_\alpha(\mC_2, \phi(\mF_2), \overline{\mC}, \overline{\mF}, \mT^\star_2) < 	\mathcal{E}^{FGW}_\alpha(\mC_2, \phi(\mF_2), \overline{\mC}, \overline{\mF}, \mP\mT^\star_1)
\end{equation}
then using the equalities $\ref{eq:invariant_FGW_2}$ for the l.h.s and $\ref{eq:invariant_FGW}$ for the r.h.s, we have
\begin{equation}
\mathcal{E}^{FGW}_\alpha(\mC_1, \phi(\mF_1), \overline{\mC}, \overline{\mF}, \mP^\top \mT^\star_2) < 	\mathcal{E}^{FGW}_\alpha(\mC_1, \phi(\mF_1), \overline{\mC}, \overline{\mF}, \mT^\star_1)
\end{equation}
which contradicts the optimality of $\mT^\star_1$. Therefore such $\mT^\star_2$ can not exist and necessarily $\mP \mT^\star_1$ is an optimal coupling for $\FGW_\alpha(\mC_2, \phi(\mF_2), \vh_2,  \overline{\mC}, \overline{\mF}, \overline{\vh})$. Finally, we can conclude using the optimality of $\mT^\star_1$ and $\mP \mT^\star_1$ for their respective FGW matching problems and the equality $\ref{eq:invariant_FGW}$:
\begin{equation}
	\begin{split}
		\mathcal{E}^{FGW}_\alpha(\mC_1, \phi(\mF_1), \overline{\mC}, \overline{\mF}, \mT^\star_1) &= \mathcal{E}^{FGW}_\alpha(\mC_2, \phi(\mF_2), \overline{\mC}, \overline{\mF}, \mP\mT^\star_1) \\
		\Leftrightarrow \FGW_\alpha(\mC_1, \phi(\mF_1), \vh_1, \overline{\mC}, \overline{\mF}, \overline{\vh}) &= \FGW_\alpha(\mC_2, \phi(\mF_2), \vh_2, \overline{\mC}, \overline{\mF}, \overline{\vh}) \\
	\end{split}
\end{equation}
\end{proof}

\section{Complements on our experimental results on synthetic datasets}
\begin{figure}[H] 
	\begin{center}
		\includegraphics[height=0.22\linewidth]{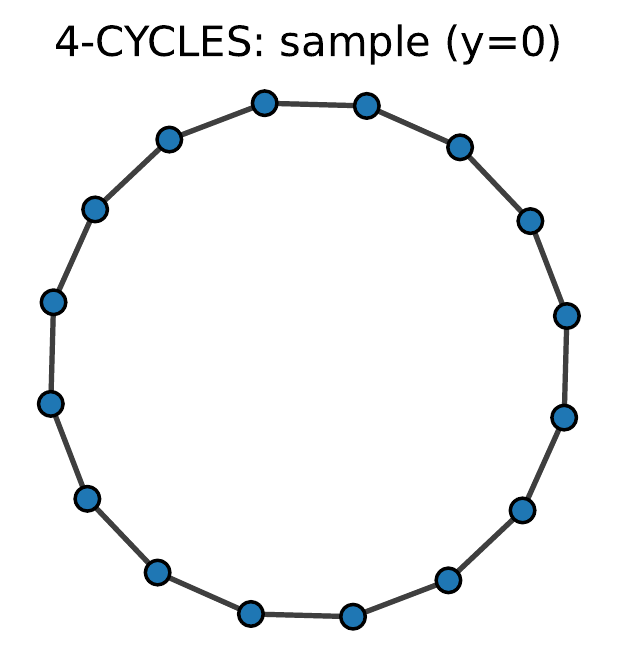} \hfill
		\includegraphics[height=0.22\linewidth]{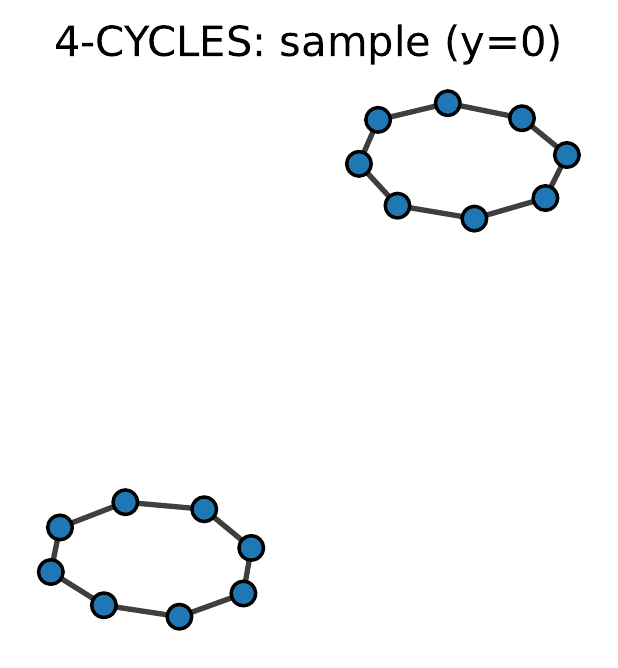}\hfill
		\includegraphics[height=0.22\linewidth]{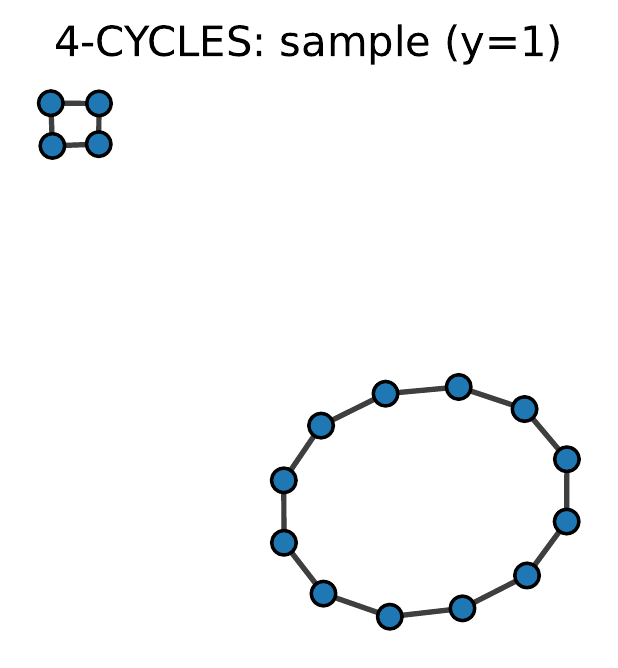}\hfill
		\includegraphics[height=0.22\linewidth]{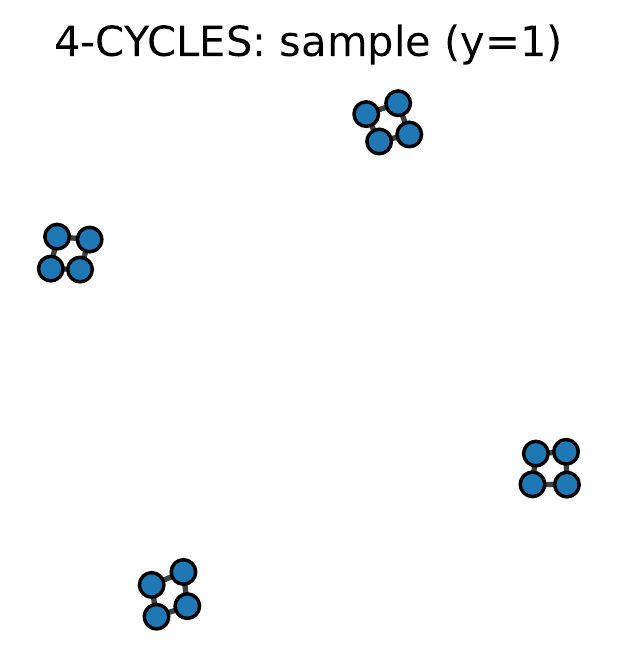}
	\end{center}
	\caption{\label{fig:4cycles_samples}Few samples with different labels $y \in \{0,1\}$ from the dataset 4-CYCLES.}
\end{figure}
\begin{figure}[H] 
	\begin{center}
		\includegraphics[height=0.22\linewidth]{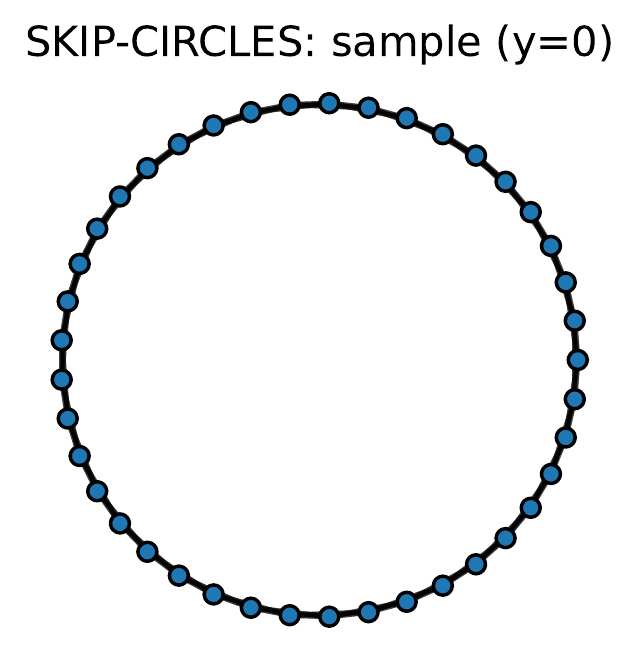} \hfill
		\includegraphics[height=0.22\linewidth]{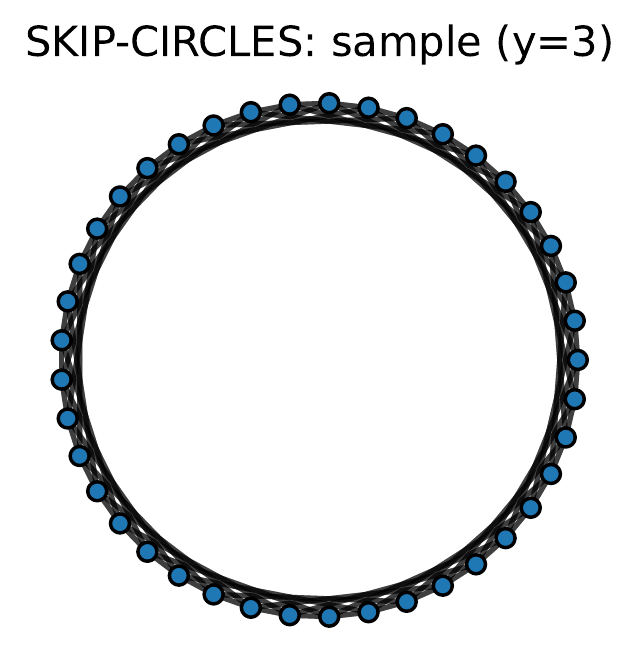} \hfill
		\includegraphics[height=0.22\linewidth]{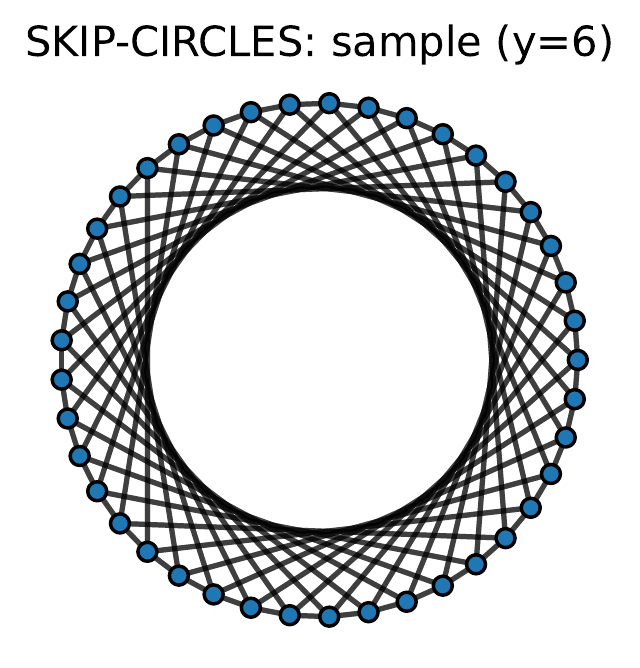} \hfill
		\includegraphics[height=0.22\linewidth]{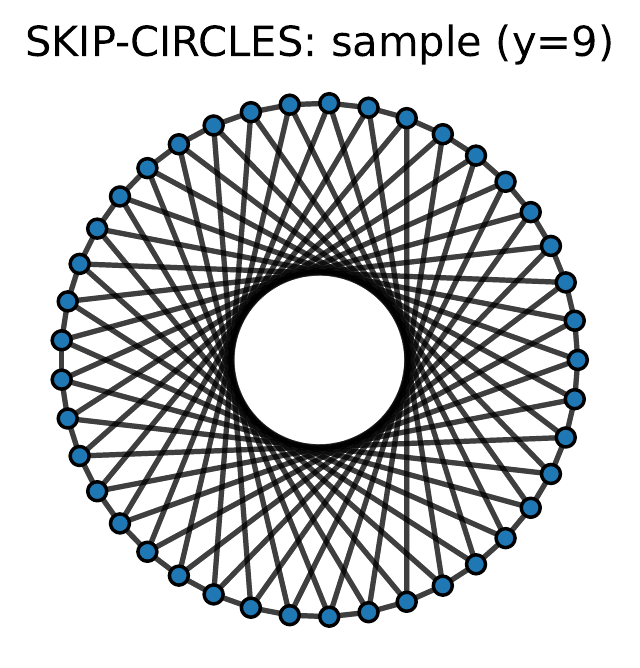} \hfill
	\end{center}
	\caption{\label{fig:skipcircles_samples}Unique sample from different labels $y \in \{0,3,6,9\}$ corresponding respectively to $\{2,5,11,16\}$ hops from the dataset SKIP-CIRCLES.}
\end{figure}
We provide here some insights and results on the synthetic datasets studied in subsection \ref{sec:beyondWL} of the main paper.
\paragraph{Datasets.} 
We considered two synthetic datasets:
\begin{itemize}
	\item 4-CYCLES \cite{loukas2019graph, papp2021dropgnn} contains graphs with (possibly) disconnected cycles where the label $y_i$ is the presence of a cycle of length 4, as illustrated in Figure \ref{fig:4cycles_samples}.
\end{itemize}
\begin{itemize}
	\item 
	SKIP-CIRCLES \cite{chen2019equivalence} contains circular graphs with skip links and the labels (10 classes) are the lengths of the skip links among $\{2, 3, 4, 5, 6, 9, 11, 12, 13, 16\}$, as illustrated in figure \ref{fig:skipcircles_samples}.
\end{itemize}

\paragraph{Details on the experiments reported in the paper.}  The experiments reported in the main paper, focus on the adjacency
matrices for $\mC_i$ for which two flavours of $\TFGW$ are investigated: 1) in $\TFGW$-fix we fix the
templates by sampling \emph{one template per class} from the training dataset (this can be seen as a simpler FGW feature extraction);
2) for $\TFGW$ we learn the templates from the training
data (as many as the number of classes). For both methods we used the FGW fixing $\alpha=1$ (i.e the GW distance) as degrees are not discriminant for these datasets. We fixed for both methods the same MLP learned to predict labels from the TFGW embeddings. This MLP ($\psi_v$) contains 2 layers of 128 hidden units each, with ReLU activations. The models are learnt for 1000 epochs using Adam optimizer with an initial learning rate of 0.01 and taking the whole train dataset as a batch. For DropGIN \cite{papp2021dropgnn} and GIN \cite{xu2018powerful} we replicated their experiments taking the same settings than the ones described by \cite{papp2021dropgnn}. For 4-CYCLES, they used 4 GIN layers composed of 2 layers each with 16 hidden units and batch normalization. For SKIP-CIRCLES, they used 9 similar GIN layers except that for each GIN layer the number of layer-wise hidden units is set to 32 instead of 16. Finally, as prescribed by \cite{papp2021dropgnn} we set the number of runs to $r=50$ and the node dropout probability $p=\frac{2}{m}$ where $m$ is the mean number of nodes in the graphs in the dataset. These methods also use the Adam optimizer with an initial learning rate of $0.01$, where the learning rate by 0.5 every 50 epochs during 1000 epochs. Switching their optimization scheme to the ones used for TFGW did not change the reported results so we kept the one from the original paper.
\begin{table}[!t]
	\centering
	\caption{Statistics on real datasets considered in our benchmark.}
	\label{tab:data_statistics}
	\scalebox{0.7}{
		\begin{tabular}{|l|r|r|r|r|r|r|r|r}
			\hline
			datasets &  features &  \#graphs & \#classes & mean \#nodes  &  min \#nodes & max \#nodes & median \#nodes\\ \hline
			MUTAG & $\{0..6\} $ & 188 & 2 & 17.93&10 & 28 & 17.5 \\ \hline
			PTC-MR& $\{0,..,17\}$ & 344 & 2 &  14.29& 2 & 64 & 13  \\ \hline
			ENZYMES& $\R^{18}$ & 600 & 6& 32.63 & 2 & 126 & 32 \\ \hline
			PROTEIN& $\R^{29}$ & 1113& 2 & 29.06& 4 & 620& 26\\ \hline
			NCI1 &$\{0,...,36\}$ & 4110& 2 & 29.87& 3 & 111& 27\\ \hline
			IMDB-B &     None &  1000   &  2 & 19.77 & 12 & 136 & 17\\ \hline
			IMDB-M &     None & 1500 & 3 & 13.00 &  7 & 89 & 10  \\ \hline
			COLLAB & None & 5000 & 3 & 74.5 & 32& 492 & 52 \\\hline
	\end{tabular}}
\end{table}

\paragraph{Additional experiments.} \begin{wrapfigure}{r}{0.3\textwidth}  \vspace{-7mm}
	\begin{center}
		\includegraphics[width=\linewidth]{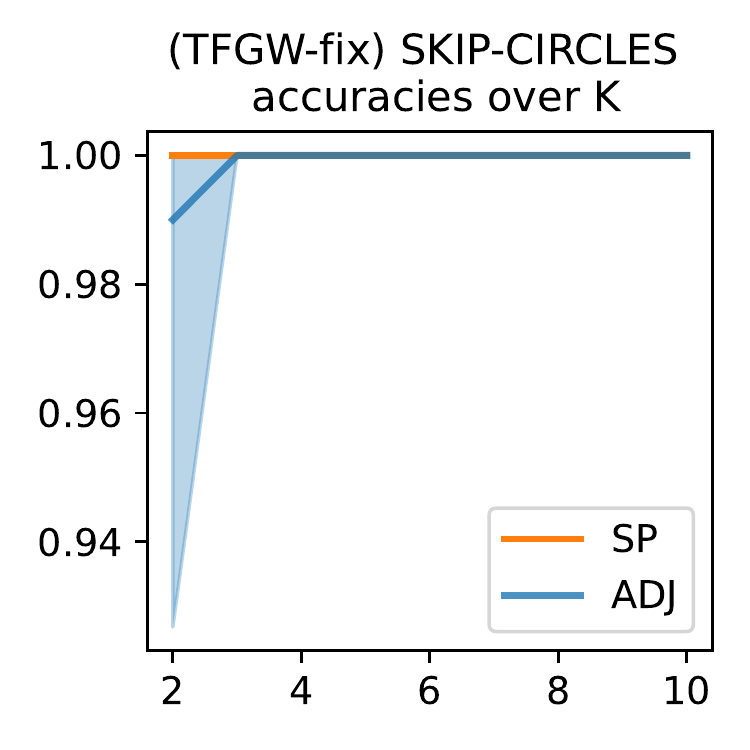}
	\end{center}
	\vspace{-5mm}
	\caption{\label{fig:skipcircles_tfgwfix} Test accuracies on SKIP-CIRCLES of TFGW-fix for $K \in \{2,..., 10\}$.} 
\end{wrapfigure}
To further emphasize the discriminative power of the TFGW embeddings, we report here additional experiments conducted on the SKIP-CIRCLES simulated datasets. For adjacency (ADJ) and shortest-path (SP) matrices as $\mC_i$, instead of using one template per class ($K=10$) as reported in the main paper for the sake of conciseness, we stress the TFGW-fix models by learning on $K \in \{2,...,10\}$ fixed templates sampled from the dataset. We report in Figure \ref{fig:skipcircles_tfgwfix}, the test accuracies averaged over 10 simulations with the same other settings than for the previously reported experiments. The averaged accuracies are illustrated in bold, while the intervals between the minimum and the maximum accuracy across runs is illustrated with a lower intensity. We can see that both methods perfectly distinguish the classes using at most 3 templates. Moreover only 2 suffice to achieve such performance using SP matrices, which is not the case for ADJ matrices. These results support our detailed analyzes in section \ref{sec:experiments} of the paper where the SP matrices are shown to better perform than ADJ ones, when no pre-processing of the node features is used.
\section{Complement on the experiments on real datasets}
We detail here few aspects of our experiments on real datasets reported in subsection \ref{sec:classif} and \ref{sec:sensiti} of the main paper. We first report some statistics on these datasets in Table \ref{tab:data_statistics}. 

\paragraph{Graph Classification benchmark.} We complete here the description of the settings and the validated hyper-parameters that we used in our benchmark whose results are reported in Table \ref{tab:benchmark} of the main paper.

For our method TFGW, we validate the number of templates $K$ in $\{\beta |\mathcal{Y}|\}_\beta$, with $\beta \in \{2,4,6,8\}$ and $|\mathcal{Y}|$ the number of classes. Only for ENZYMES with 6 classes of 100 graphs each, we validate $\beta \in \{1, 2, 3, 4\}$.
All parameters of our $\TFGW$ layers are learned, namely the templates structure $\overline{\mC}_k$, feature matrix $\overline{\mF}_k$, the weights $\overline{\vh}_k$, and finally a single trade-off parameter $\alpha$. Moreover these templates are initialized by randomly sampling from the train dataset, a same number of graphs for each class. We also learn $\phi_\vu$ taken as a GIN architecture
\cite{xu2018powerful} composed of $L=2$ GIN layers aggregated using the Jumping Knowledge (concatenation) scheme \cite{xu2018representation}. Every GIN layer is a MLP of 2 layers with batch normalization, whose number of units is validated in $\{16, 32\}$ for bioinformatics datasets and fixed to 64 for social network datasets, as in \cite{xu2018powerful}. For predictions, the same MLP $\psi_v$ than for the experiments on synthetic datasets is used. Finally a dropout technique is applied to $\psi_v$ with a rate validated in $\{0, 0.2, 0.5\}$. We learn our models over 500 epochs using Adam optimizer with an initial learning rate of $0.01$ and a batch size of 128.

For OT-GNN \cite{chen2020optimal} which is the approach the most similar to TFGW, the exact same settings and validated hyper-parameters are considered. For WEGL \cite{kolouri2020wasserstein}, we validated the number of layers $L \in \{1,2,3,4\}$ for their non-parametric embeddings, either using the final embeddings or the concatenation of embeddings per layer, then learnt Random Forest classifiers shown to achieve the best results on average across datasets, whose hyper-parameters are validated as in the original paper. For GIN \cite{xu2018powerful} and DropGIN \cite{papp2021dropgnn} which share the same architecture than TFGW except for the pooling strategy (sum aggregation instead of FGW distances), we also validated the same hyper-parameters than for TFGW. For Patchy-SAN \cite{niepert2016learning}, we validated their receptive field parameter in $k \in \{5,10,10^E\}$ considering the same other settings than the authors. For DIFFPOOL \cite{ying2018hierarchical}, following the discussion of the authors, we validated for their default version the number of pooling layer in $\{1,2\}$, the clustering ratios in $\{10\%, 25\%\}$. Finally for the kernel methods, we cross validated the SVM parameters $C \in \{10^{-7}, 10^{-6},..., 10^7\}$ and $\gamma \in \{ 2^{-10}, 2^{-9},..., 2^{10}\}$ using the scikit-learn implementation \cite{sklearn_api}. Then for the FGW kernel \cite{titouan2019optimal}, we validated 15 values of the trade-off parameter $\alpha$ via a logspace search in $(0, 0.5)$ and symmetrically
$(0.5, 1)$. For WL \cite{shervashidze2011weisfeiler} and WWL \cite{Togninalli19}, the number of WL step is validated in $\{1,...,10\}$ while taking the discrete and continuous versions of the WL refinement suggested in \cite{Togninalli19}.

 \begin{figure}[t!]
 	\begin{center}
 		\includegraphics[height=3cm]{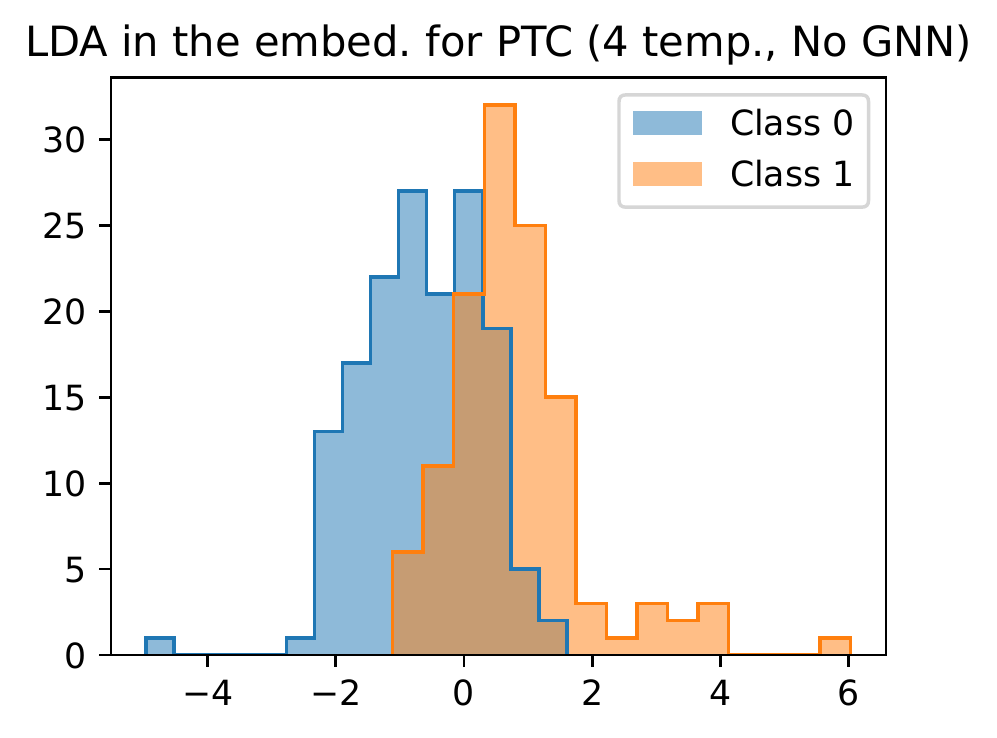}\hspace{5mm}
 		\includegraphics[height=3cm]{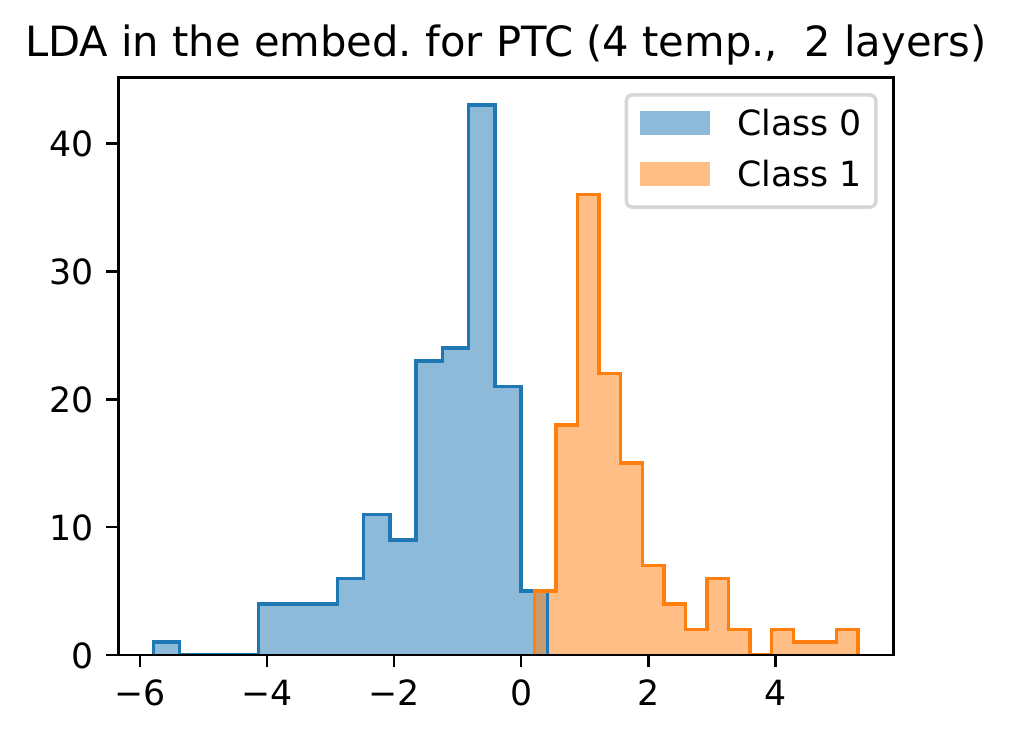}\hspace{5mm}
 		\includegraphics[height=3cm]{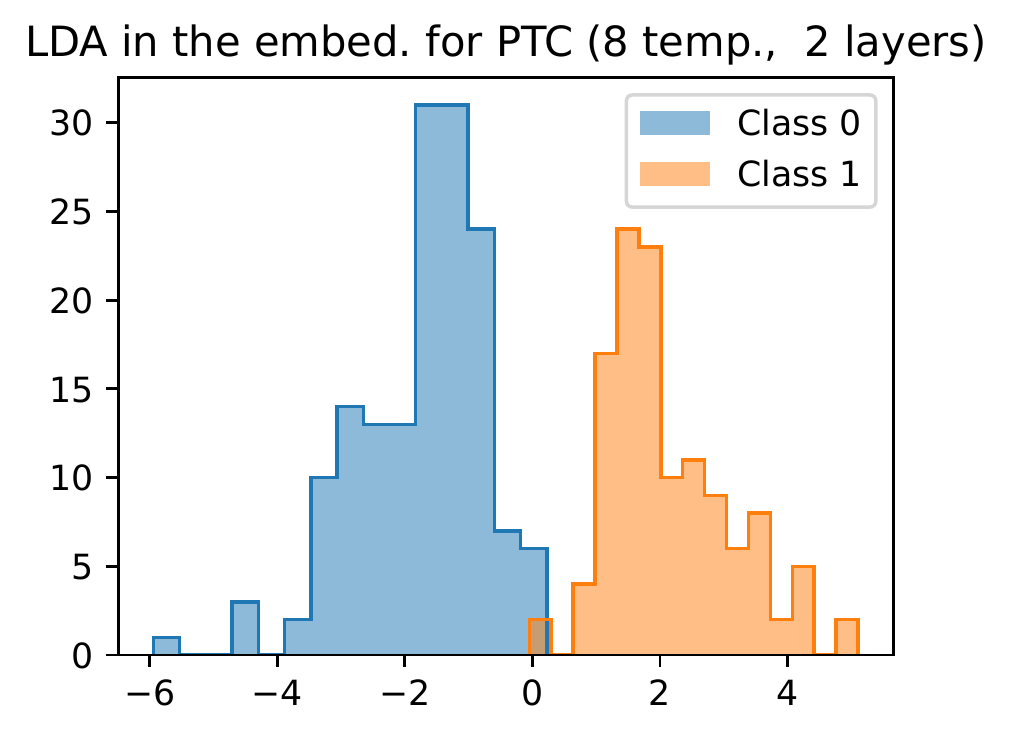}
 	\end{center}\vspace{-2mm}
 	\caption{LDA 1D projections of the distance embeddings for different
 		models learned on PTC.\label{fig:lda} }
 \end{figure}
 \begin{figure}[t!]
 	\begin{center}
 		\includegraphics[height=3.5cm]{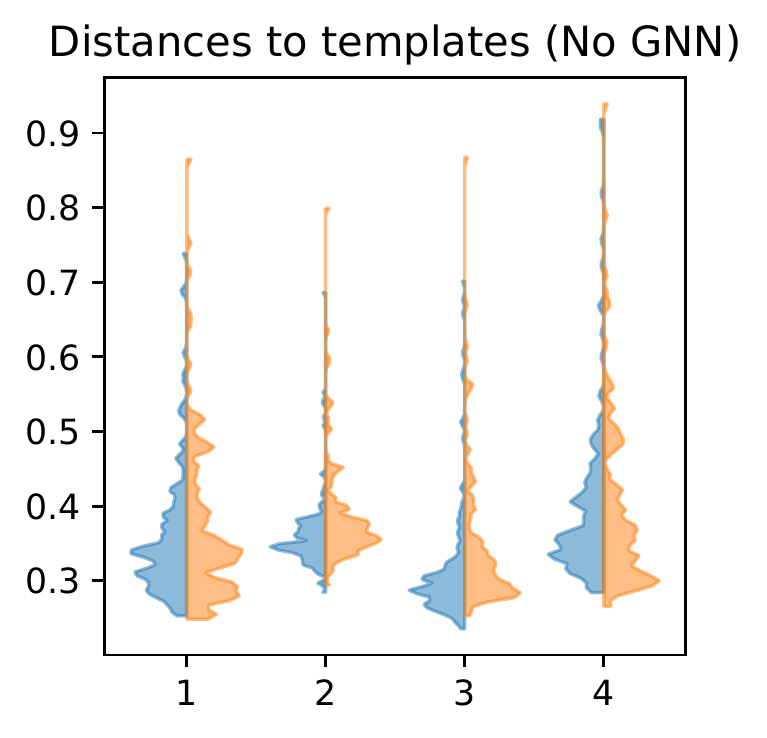}
 		\includegraphics[height=3.5cm]{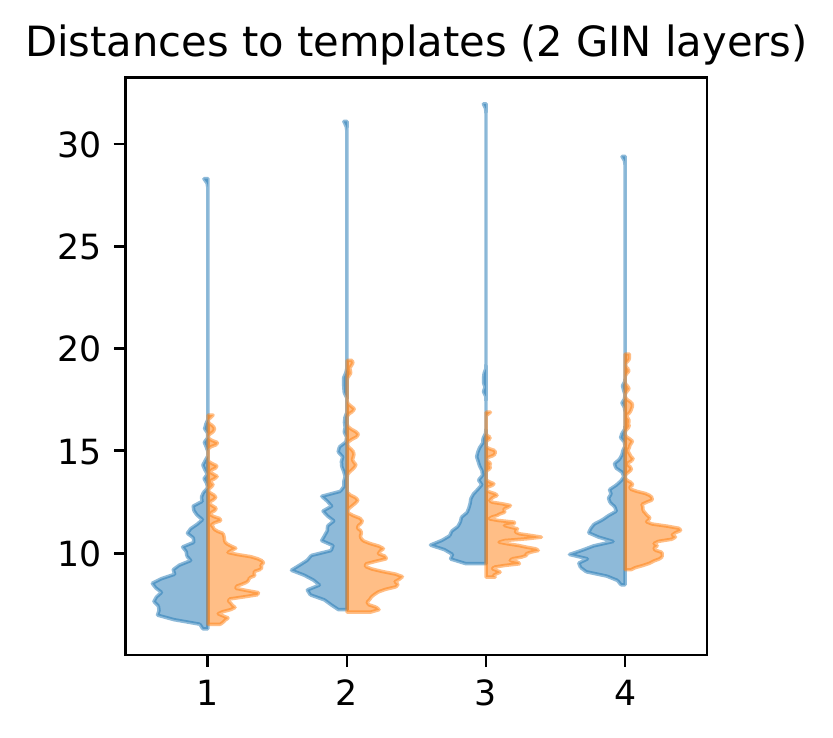}
 		\includegraphics[height=3.5cm]{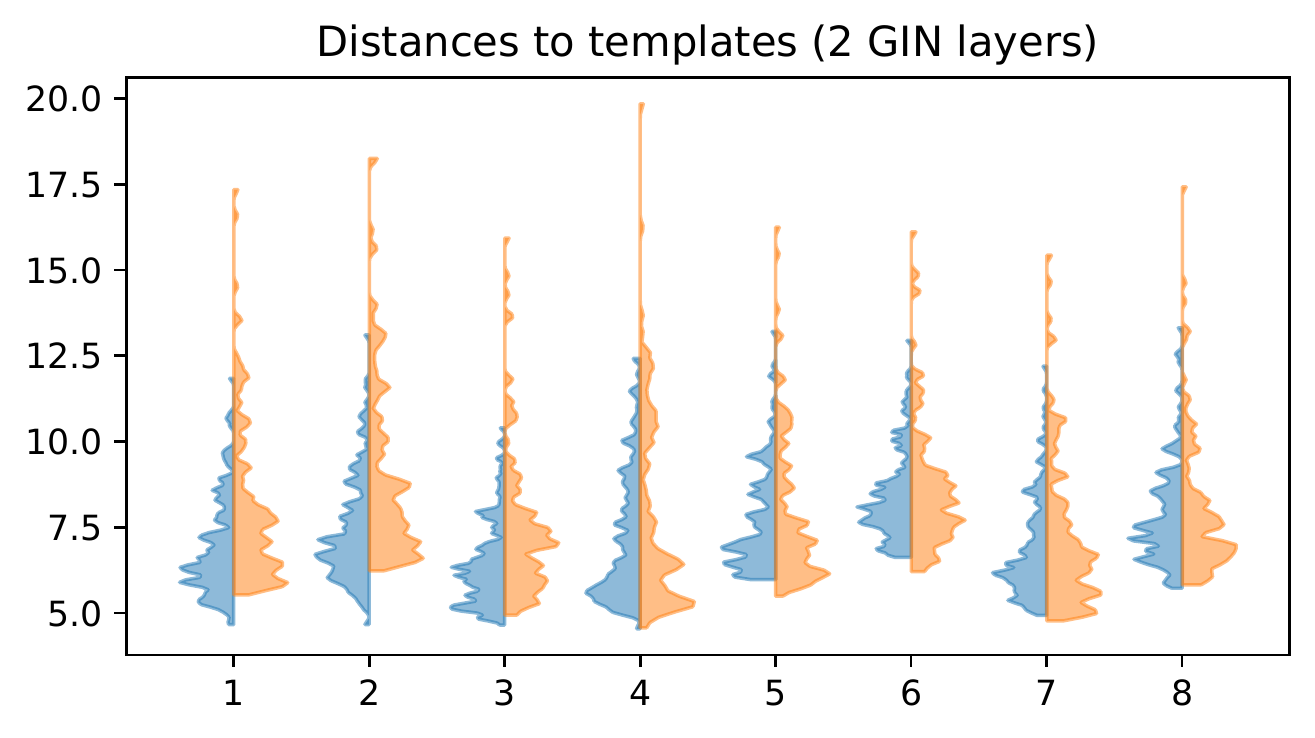}
 	\end{center}\vspace{-2mm}
 	\caption{Distributions of the distance to the templates fo each templates
 		for different models learned on PTC.\label{fig:violin} }
 \end{figure}
 
\paragraph{Additional visualization of the TFGW embeddings.} In this paragraph we complete the analysis of our TFGW embeddings detailed in the subsection \ref{sec:sensiti} of the main paper. Especially, we illustrate the LDA (Figure \ref{fig:lda}) and the distributions (Figure \ref{fig:violin}) of our distance embeddings learned on PTC with $L=0$ and $L=2$, and the number of templates $K$ varying in $\{4, 8\}$. Note that these embeddings are the same than studied thanks to a PCA in the main paper. First, the figure \ref{fig:lda} supports our conclusions regarding the separability of our embeddings, which becomes more linear when increasing the number of GIN layers from $L=0$ to $L=2$ and the number of templates from $K=4$ to $K=8$. Then, the distribution of the distances in figure \ref{fig:violin} exhibits that the discrimination between samples of different classes is achieved through the modes of these distributions. One can also notice that the range of distances increases considerably from $L=0$ to $L=2$. This coincides with the fact that the learned templates are extreme points in the embedding, as illustrated in the main paper, which might encode "exaggerated" features in order to maximize the margin between classes in the embedding. An instance of such learned templates are illustrated in figure \ref{fig:templates}. By comparing these templates (on the left) with samples from the dataset (on the right), we can clearly see that the learned templates do not represent realistic graphs from the data. Such behavior was to be expected in our end-to-end framework where the prediction task is achieved by a MLP with non-linear activations. However we believe that our promising results achieved thanks to our $\TFGW$ modeling can open the door to novel and
hopefully more interpretable end-to-end architectures.
\begin{figure}[!t]
	\begin{center}
			\includegraphics[height=4.5cm]{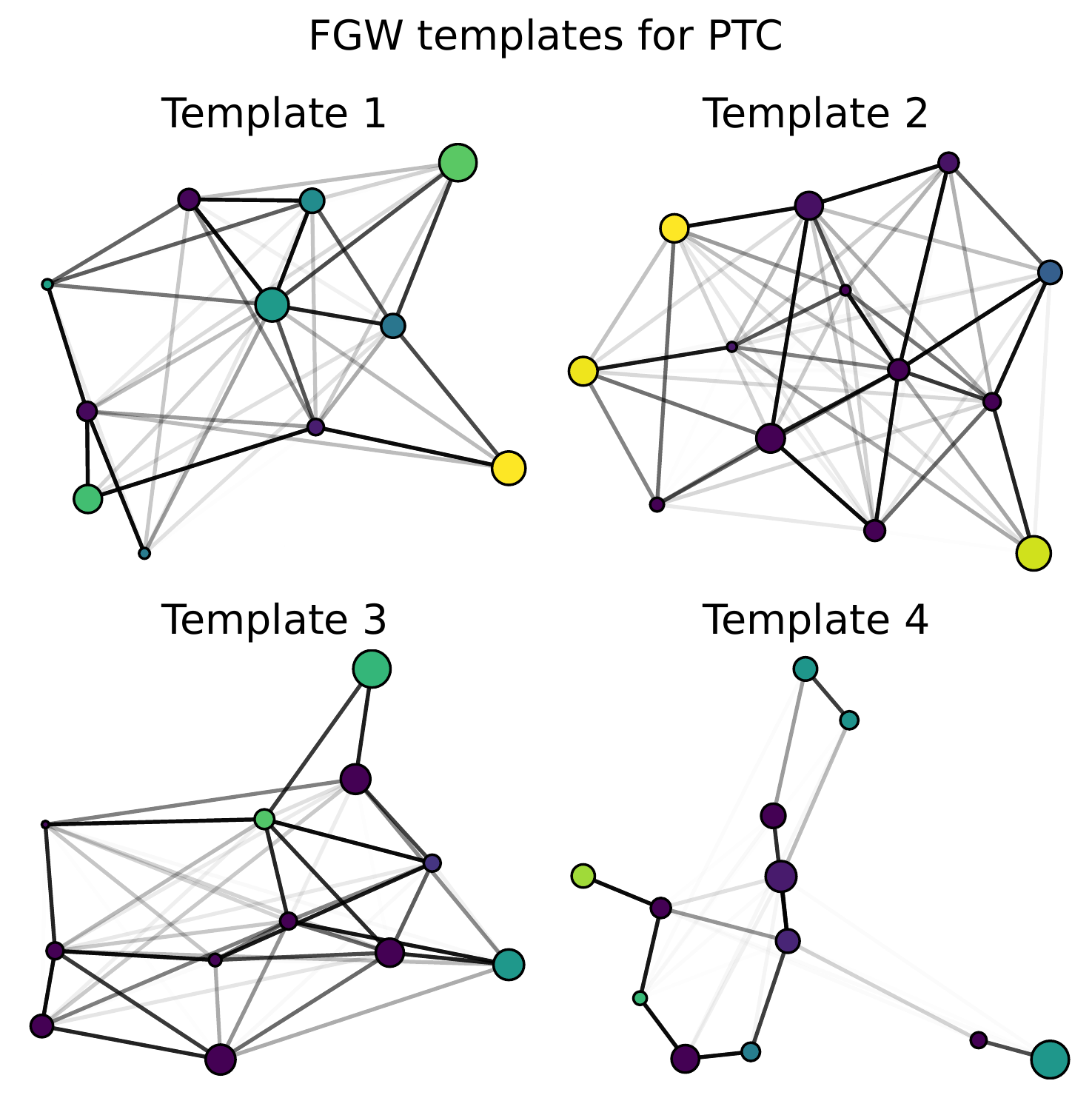} \hfill \includegraphics[height=4.5cm]{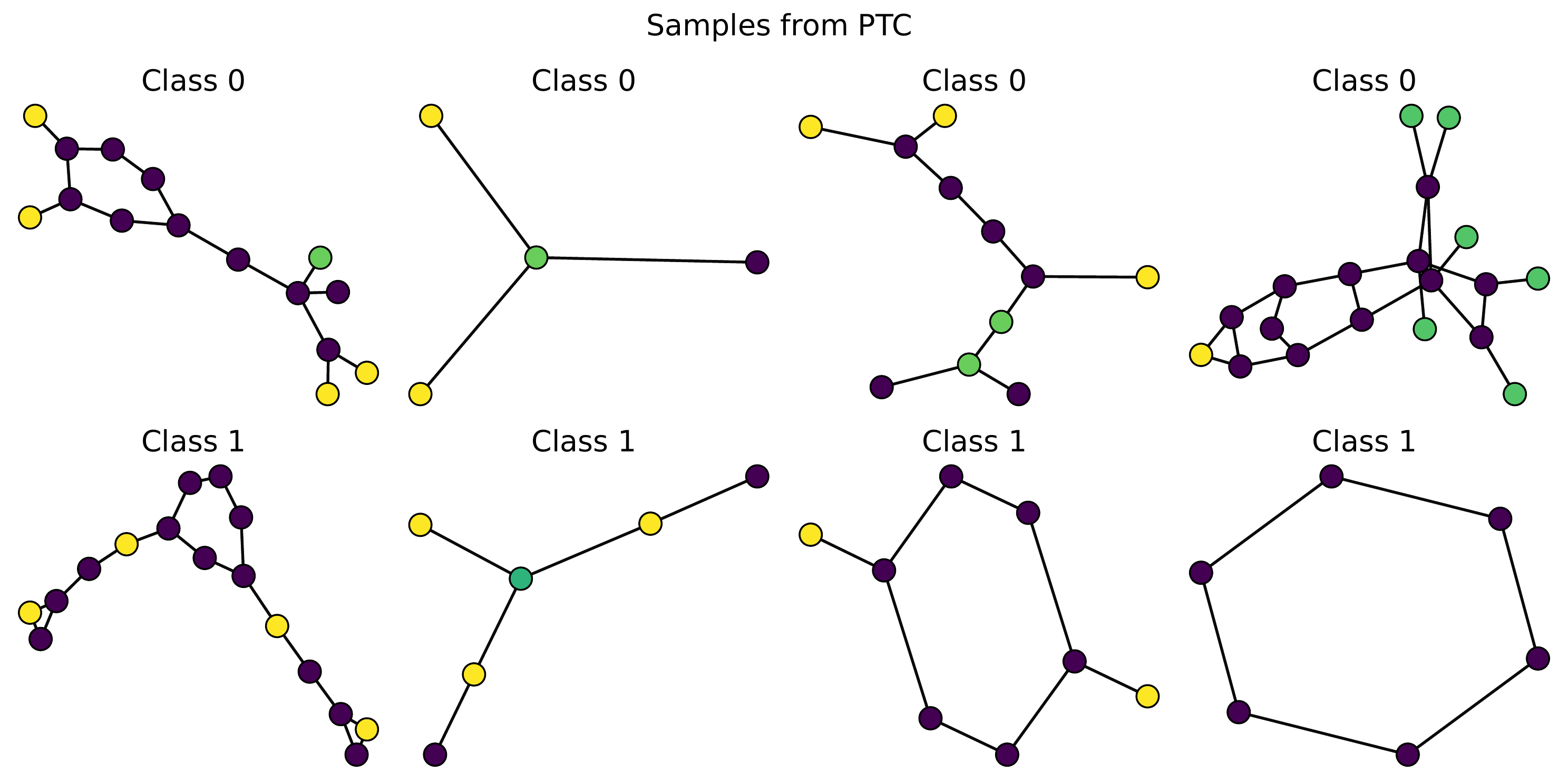}
	\end{center}
\caption{\label{fig:templates} Illustration of the templates learned on PTC with $K=4$ and $L=0$ (on the left side), and some samples from this dataset (on the right sight). The graph structures are represented using the entries of $\overline{\mC}_k$ (resp. $\mC_i$) as repulsive strength and the corresponding edges are colored in shades of grey (black being the maximum). The node colors are computed based on their features $\overline{\mF}_k$ (resp. $\mF_i$). The nodes size are made proportional to the weights $\overline{\vh}_k$ (resp. $\vh_i$).}
\end{figure}

\paragraph{Sensitivity of GIN to the number of layers.} \begin{wrapfigure}{r}{0.5\textwidth}  \vspace{-2mm}
	\begin{center}
		\includegraphics[width=\linewidth]{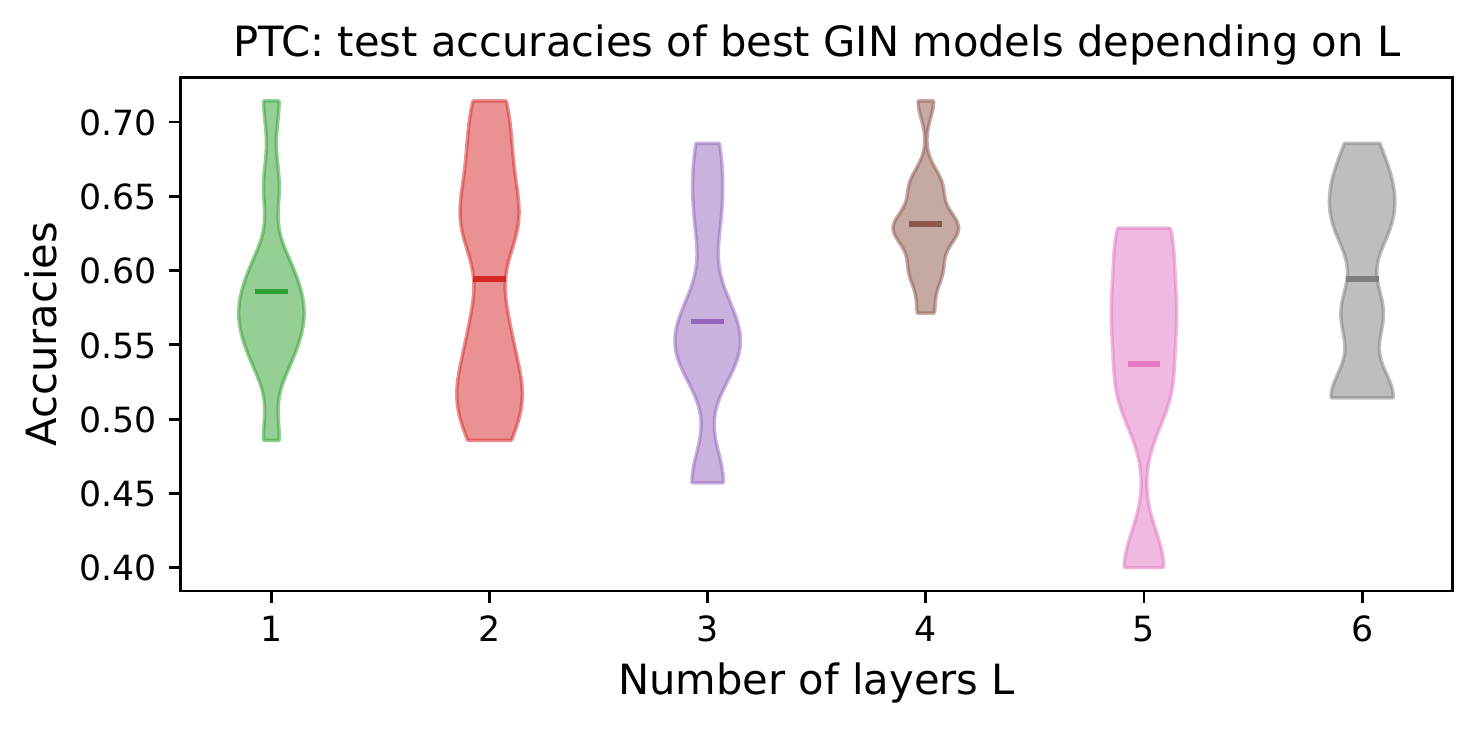}
	\end{center}
	\vspace{-5mm}
	\caption{\label{fig:ptc_gin} Test accuracies on PTC of GIN for $L \in \{1,..., 6\}$.} 
\end{wrapfigure}We aim here to benchmark our sensitivity analysis to the number of templates and the number of GIN layers of TFGW, reported in the subsection \ref{sec:sensiti} of the main paper. 
 To this end, we learned GIN models with a number of GIN layers varying in $L \in \{1,2,...,6\}$, using analog settings and validation than detailed in our graph classification benchmark. The test accuracies of the validated models for each $L$ are reported in Figure \ref{fig:ptc_gin}. First, the model with $L=4$ (default for the method \cite{xu2018powerful}) leads to best performances on average. Then, no clear pattern of overfitting is observed for $L>4$, as indeed $L=5$ leads to worst performances but $L=6$ leads to the second best model in this benchmark. Such behavior may come from the Jumping Knowledge scheme (with concatenation) \cite{xu2018representation} as argued by the authors.

	\label{sec:supplementary}
	
\end{document}